\documentclass{article}

\usepackage{microtype}
\usepackage{graphicx}
\usepackage{subfigure}
\usepackage{booktabs}

\usepackage{hyperref}

\usepackage[accepted]{icml2023}

\usepackage{amsmath}
\usepackage{amssymb}
\usepackage{mathtools}
\usepackage{amsthm}
\usepackage{bbm}
\usepackage[utf8]{inputenc} 
\usepackage[T1]{fontenc}    
\usepackage{url}            
\usepackage{nicefrac}       
\usepackage{xcolor}         

\usepackage{multirow}
\usepackage{mathtools}
\usepackage{algorithm}
\usepackage{wrapfig}

\def\*#1{\mathbf{#1}}
\usepackage[shortlabels]{enumitem}

\usepackage[capitalize,noabbrev]{cleveref}

\theoremstyle{plain}
\newtheorem{theorem}{Theorem}[section]

\newtheorem{lemma}[theorem]{Lemma}

\theoremstyle{definition}
\newtheorem{definition}[theorem]{Definition}

\theoremstyle{remark}

\usepackage[textsize=tiny]{todonotes}

\usepackage{xspace}
\usepackage{array}
\usepackage{ragged2e}
\newcolumntype{P}[1]{>{\RaggedRight\hspace{0pt}}p{#1}}
\newcolumntype{X}[1]{>{\RaggedRight\hspace*{0pt}}p{#1}}
\usepackage{tcolorbox}
\makeatletter
\@namedef{ver@everyshi.sty}{}
\makeatother
\usepackage{tikz}
\usetikzlibrary{arrows,shapes,positioning,shadows,trees,mindmap}
\usepackage[edges]{forest}
\usetikzlibrary{arrows.meta}
\colorlet{linecol}{black!75}
\usepackage{xkcdcolors} 

\usepackage{tikz}
\usetikzlibrary{backgrounds}
\usetikzlibrary{arrows,shapes}
\usetikzlibrary{tikzmark}
\usetikzlibrary{calc}
\newcommand{\highlight}[2]{\colorbox{#1!17}{$\displaystyle #2$}}

\renewcommand{\highlight}[2]{\colorbox{#1!17}{#2}}

\newcommand{\cube}[1]{%
\scalebox{#1}{
\begin{tikzpicture}
\pgfmathsetmacro{\cubex}{0.2}
\pgfmathsetmacro{\cubey}{0.2}
\pgfmathsetmacro{\cubez}{0.2}
\draw (0,0,0) -- ++(-\cubex,0,0) -- ++(0,-\cubey,0) -- ++(\cubex,0,0) -- cycle;
\draw (0,0,0) -- ++(0,0,-\cubez) -- ++(0,-\cubey,0) -- ++(0,0,\cubez) -- cycle;
\draw (0,0,0) -- ++(-\cubex,0,0) -- ++(0,0,-\cubez) -- ++(\cubex,0,0) -- cycle;
\end{tikzpicture}
}
}
\newcommand{\sphere}[2]{%
\scalebox{#1}{
\begin{tikzpicture}
  \shade[ball color = #2!40, opacity = 0.2] (0,0) circle (0.2cm);
  \draw (0,0) circle (0.2cm);
\end{tikzpicture}
}
}

\newcommand{\cylinder}[1]{%
\scalebox{#1}{
\begin{tikzpicture}
\node[cylinder, draw, shape border rotate = 90,minimum size = 0.4cm] (c) at (0,0) {};
\end{tikzpicture}
}
}

\let\classAND\AND
\let\AND\relax
\usepackage{algorithmic}

\let\AND\classAND
\AtBeginEnvironment{algorithmic}{\let\AND\algoAND}

\floatname{algorithm}{Algorithm}
\renewcommand{\algorithmicrequire}{\textbf{Input:}}
\renewcommand{\algorithmicensure}{\textbf{Output:}}
\renewcommand{\algorithmicendif}{\textbf{end}}

\icmltitlerunning{When and How Does Known Class Help Discover Unknown Ones? }

\begin{document}

\twocolumn[
\icmltitle{When and How Does Known Class Help Discover Unknown Ones? \\Provable Understanding Through Spectral Analysis}

\icmlsetsymbol{equal}{*}

\begin{icmlauthorlist}
\icmlauthor{Yiyou Sun}{wisc}
\icmlauthor{Zhenmei Shi}{wisc}
\icmlauthor{Yingyu Liang}{wisc}
\icmlauthor{Yixuan Li}{wisc}
\end{icmlauthorlist}

\icmlaffiliation{wisc}{Computer Sciences Department, University of Wisconsin-Madison}

\icmlcorrespondingauthor{Yiyou Sun, Yixuan Li}{{sunyiyou, sharonli@cs.wisc.edu}}

\icmlaffiliation{wisc}{Department of Computer Sciences, University of Wisconsin - Madison}

\icmlkeywords{Machine Learning, ICML}

\vskip 0.3in
]

\printAffiliationsAndNotice{}  

\begin{abstract}
Novel Class Discovery (NCD) aims at inferring novel  classes in an unlabeled set by leveraging prior knowledge from a labeled set with known classes. Despite its importance, there is a lack of theoretical foundations for NCD. This paper bridges the gap by providing an analytical framework to formalize and investigate \emph{when and how known classes can help discover novel classes}. Tailored to the NCD problem, we introduce a graph-theoretic representation that can be learned by a novel NCD Spectral Contrastive Loss (NSCL). Minimizing this objective is equivalent to factorizing the graph's adjacency matrix, which allows us to derive a provable error bound and provide the sufficient and necessary condition for NCD. Empirically, NSCL can match or outperform several strong baselines on common benchmark datasets, which is appealing for practical usage while enjoying theoretical guarantees. Code is available at:~\url{https://github.com/deeplearning-wisc/NSCL.git}.
\end{abstract}

\section{Introduction}
\label{sec:intro}

Though modern machine learning methods have achieved remarkable success~\citep{he2016deep, chen2020simclr, song2020score, wang2022pico}, the vast majority of learning algorithms have been driven by the closed-world setting, where the classes are assumed stationary and unchanged between training and testing. 
However, machine learning models in the open world will inevitably encounter novel classes that are outside the existing known categories~\cite{sun2021react,sun2022out,ming2022delving,ming2023exploit}.  Novel Class Discovery (NCD)~\cite{Han2019dtc} has emerged as an important problem, which aims to cluster similar samples in an unlabeled dataset (of novel classes) by way of utilizing knowledge from the labeled data (of known classes).  Key to NCD  is harnessing the power of labeled data for possible knowledge sharing and transfer to the unlabeled data~\cite{hsu2017kcl, Han2019dtc, hsu2019mcl, zhong2021openmix, zhao2020rankstat, yang2022divide, sun2023opencon}.

\begin{figure}[t]
    \centering
    \includegraphics[width=0.99\linewidth]{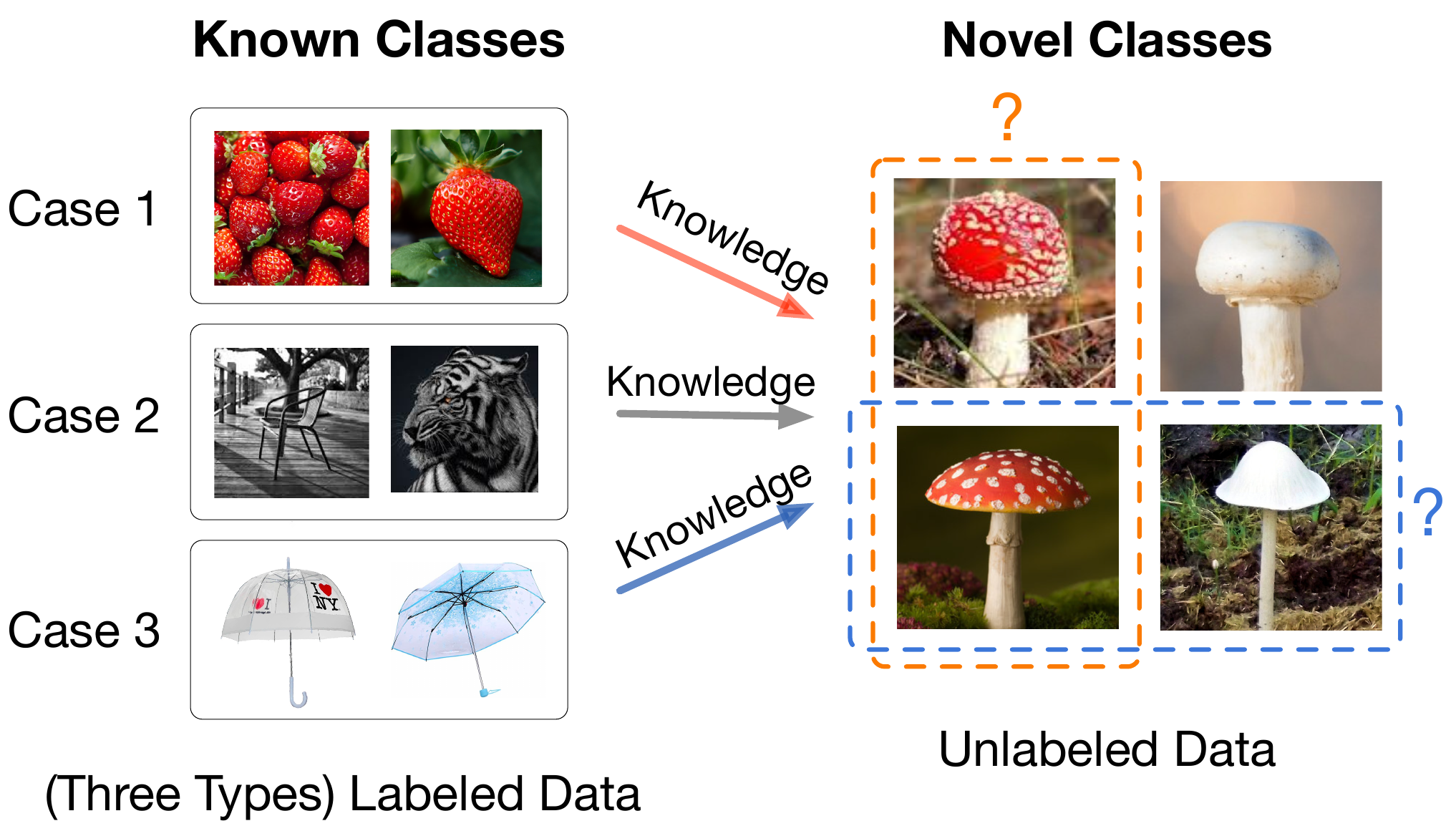}
    \caption{\textbf{ Novel Class Discovery (NCD)} aims to cluster similar samples in unlabeled
data (right), by way of utilizing knowledge from the labeled data (left). We illustrate scenarios where different known classes  could result in different novel clusters (e.g., red mushrooms or mushrooms with umbrella shapes). This paper aims to provide a formal understanding.
    }
    \label{fig:teaser}
\end{figure}

One promising approach for NCD is to  learn feature representation jointly from both labeled and unlabeled data, so that meaningful cluster structures emerge as novel classes. We argue that interesting intricacies can arise in this learning process---the resulting novel clusters may be very different, depending on the type of known class provided. We exemplify the nuances in Figure~\ref{fig:teaser}. In one scenario, the novel class ``red mushroom'' can be discovered, provided with the known class ``strawberry'' of a shared color feature. Alternatively, a different novel class can also emerge by grouping the bottom two images together (as ``mushroom with umbrella shape'' class), if the umbrella-shape images are given as a known class to the learner. We argue---perhaps obviously---that a formalized understanding of the intricate phenomenon is needed. This motivates our research:
\begin{center}
   \textit{\textbf{When and  how does the known class help discover novel classes?}}  
\end{center}

Despite the empirical successes in recent years, there is a limited theoretical understanding and formalization for novel class discovery. To the best of our knowledge, there is no prior work that investigated this research question from a rigorous theoretical standpoint or provided provable error bound. Our work thus complements the existing works by filling in the critical blank.

In this paper, we start by formalizing a new learning algorithm that facilitates the understanding of NCD from a spectral analysis perspective. Our theoretical framework first introduces a graph-theoretic representation tailored for NCD, where the vertices are all the labeled and unlabeled data points, and classes form connected sub-graphs (Section~\ref{sec:graph_rep}). 
Based on this graph representation, we then 
introduce a new loss called NCD Spectral Contrastive Loss (NSCL) and show that minimizing our loss is equivalent to performing spectral decomposition on the graph (Section~\ref{sec:ncd-scl}). Such  equivalence allows us to derive the formal error bound for NCD based on the properties of the graph, which directly encodes the relations between known and novel classes.

We analyze the NCD quality by the linear probing performance on novel data, which is the least error of all possible linear classifiers with the learned representation. 
Our main result (Theorem~\ref{th:no_approx}) suggests that the linear probing  error can be significantly reduced (even to 0) when the linear span of known samples' feature covers the ``ignorance space'' of unlabeled data in discovering novel classes. 
Lastly, we verify that our theoretical guarantees can translate into empirical effectiveness. In particular, NSCL establishes competitive performance on common NCD benchmarks, outperforming the best baseline by \textbf{10.6}\% on the CIFAR-100-50 dataset (with 50 novel classes). 

Our \textbf{main contributions} are: 
\vspace{-0.2cm}
\begin{enumerate}
    \item We provide the first provable framework for the NCD problem, formalizing it by spectral decomposition of the graph containing both known and novel data. 
    Our framework  allows the research community to gain insights from a graph-theoretic perspective. 
    \vspace{-0.2cm}
    \item We propose a new loss called NCD Spectral Contrastive Loss
(NSCL) and show that minimizing our loss is equivalent to performing singular decomposition on the graph. The loss leads to strong empirical performance while enjoying theoretical guarantees.
    \vspace{-0.2cm}
    \item We provide theoretical insight by formally defining the semantic relationship between known and novel classes. Based on that, we derive an error bound of novel class discovery  and investigate the sufficient and necessary conditions for the perfect discovery results.
\end{enumerate}

\section{Related Work}
\label{sec:related}

\textbf{Novel class discovery.} 
Early works tackled novel category discovery (NCD) as a transfer learning problem, such as DTC~\citep{Han2019dtc}, KCL~\citep{hsu2017kcl}, MCL~\citep{hsu2019mcl}. 
Many subsequent works incorporate representation learning for NCD, including  RankStats~\citep{zhao2020rankstat}, NCL~\citep{zhong2021ncl} and UNO~\citep{fini2021unified}. CompEx~\cite{yang2022divide} further uses a novelty detection module to better separate novel and known. However, none of the previous works theoretically analyzed the key question: \textit{when and how do known classes help?} \citet{li2022closer} try to answer this question from an empirical perspective by comparing labeled datasets from different levels of semantic similarity. \citet{chi2021meta} directly define a solvable condition for the NCD problem but do not investigate the semantic relationship between known and novel classes. Our paper is the first work that systematically investigates the ``{when and how}'' questions by modeling the sample relevance
from a graph-theoretic perspective and providing a provable error bound for the NCD problem.

\textbf{Spectral graph theory.} 
Spectral graph theory is a classic research problem~\cite{chung1997spectral,cheeger2015lower,kannan2004clusterings,lee2014multiway,mcsherry2001spectral}, which aims to partition the graph by studying the eigenspace of the adjacency matrix. The spectral graph theory is also widely applied in machine learning~\cite{ng2001spectral,shi2000normalized,blum2001learning,zhu2003semi,argyriou2005combining,shaham2018spectralnet}. Recently, ~\citet{haochen2021provable} derive a spectral contrastive loss from the factorization of the graph's adjacency matrix which facilitates theoretical study in unsupervised domain adaptation~\cite{shen2022connect,haochen2022beyond}.
The graph definition in existing works is purely formed by the unlabeled data, whereas {our graph and adjacency matrix is uniquely tailored for the NCD problem setting and consists of both labeled data from known classes and unlabeled data from novel classes}. We offer new theoretical guarantees and insights based on the relations between known and novel classes, which has not been explored in the previous literature.%

\textbf{Theoretical analysis on contrastive learning.} 
 Recent works have advanced contrastive learning with empirical success~\cite{chen2020simclr,khosla2020supervised,zhang2021supporting,wang2022pico}, which necessitates  a theoretical foundation. ~\citet{arora2019theoretical,lee2021predicting,tosh2021contrastive,tosh2021contrastive2,balestriero2022contrastive, shi2023the} provided provable guarantees on the representations learned by contrastive learning for linear probing. ~\citet{shen2022connect,haochen2021provable,haochen2022beyond} further modeled the pairwise relation from the graphic view and provided error analysis of the downstream tasks. However, the existing body of work has mostly focused on \emph{unsupervised learning}. There is no prior theoretical work considering the NCD problem where both labeled and unlabeled data are presented. In this paper, we systematically investigate how the label information can change the representation manifold and affect the downstream novel class discovery task. 

\section{Setup}
\label{sec:setup}

Formally, we describe the data setup and learning goal for novel class discovery (NCD).

\noindent \textbf{Data setup.} We consider the empirical training set  $\mathcal{D}_{l} \cup \mathcal{D}_{u}$ as a union of labeled and unlabeled data. The labeled dataset is given by $\mathcal{D}_{l} = \{(\bar{x}_1,y_1),\ldots,(\bar{x}_i,y_i),\ldots\}$, where $y_i$ belongs to known class space $\mathcal{Y}_l$; and the unlabeled dataset is $\mathcal{D}_{u} = \{\bar{x}_1, \ldots,\bar{x}_j,\ldots\}$. We assume that each unlabeled sample $\bar x \in \mathcal{D}_u$ belongs to one of the \textbf{novel} classes, \emph{which do not overlap with the {known} classes $\mathcal{Y}_l$}.  We use $\mathcal{P}_{l}$ and $\mathcal{P}_{u}$ to denote the  marginal distributions of labeled and unlabeled data in the input space. Further, we let $\mathcal{P}_{l_i}$ denote the distribution of labeled samples with class label $i \in \mathcal{Y}_l$.

\noindent \textbf{Learning goal.} We assume that there exists an underlying class space $\mathcal{Y}_{u} = \{1, ..., |\mathcal{Y}_u|\}$ for unlabeled data $\mathcal{X}_u$, which is not revealed to the learner. The goal of novel class discovery is to learn a clustering for the novel data, which can be mapped to $\mathcal{Y}_{u}$ with low  error. %

%

\section{Spectral Contrastive Learning for Novel Class Discovery}
\label{sec:method}

In this section, we introduce a new learning algorithm for NCD, from a graph-theoretic perspective. NCD is inherently a clustering problem---grouping similar points in unlabeled data $\mathcal{D}_u$ into the same cluster, by way of possibly utilizing helpful information from the labeled data $\mathcal{D}_l$. This clustering process can be fundamentally modeled by a graph, where the vertices are all the data points and classes form connected sub-graphs. Our novel framework first introduces a graph-theoretic representation for NCD, where edges connect similar data points (Section~\ref{sec:graph_rep}). We then 
propose a new loss that performs spectral decomposition on the similarity graph and can be
 written as a contrastive learning objective on neural net representations (Section~\ref{sec:ncd-scl}).   %

\subsection{Graph-Theoretic Representation for NCD} 
\label{sec:graph_rep}

We start by formally defining the augmentation graph and adjacency matrix. 
For notation clarity, we use $\bar x$ to indicate the natural sample (raw inputs without augmentation). Given an $\bar x$, we use $\mathcal{T}(x|\Bar{x})$ to denote the probability of $x$ being augmented from $\Bar{x}$. For instance, when $\Bar{x}$ represents an image, $\mathcal{T}(\cdot|\Bar{x})$ can be the distribution of common augmentations such as Gaussian blur, color distortion, and random cropping. 
The augmentation allows us to define a general population space $\mathcal{X}$, which contains all the original images along with their augmentations. In our case, $\mathcal{X}$ ($|\mathcal{X}|=N$) is composed of two parts $\mathcal{X}_l$ ($|\mathcal{X}_l| = N_l$), $\mathcal{X}_u$ ($|\mathcal{X}_u| = N_u$) which represents the division into labeled data with known classes and unlabeled data with novel classes respectively. 
Unlike unsupervised learning~\cite{chen2020simclr}, NCD has access to both labeled and unlabeled data. This leads to two cases where two samples $x$ and $x^+$ form a {\textbf{positive pair}} if: 
\begin{enumerate}[(a)]
    \item 
    $x$ and $x^+$ are augmented from the same unlabeled image $\Bar{x}_u\sim \mathcal{P}_u$.
    
    \item $x$ and $x^+$ are augmented from two labeled samples $\Bar{x}_l$ and $\Bar{x}'_l$ \emph{with the same known class $i$}. In other words, both $\Bar{x}_l$ and $\Bar{x}'_l$ are drawn independently from $\mathcal{P}_{l_i}$.%
\end{enumerate}
We define the graph $G(\mathcal{X}, w)$ with vertex set $\mathcal{X}$ and edge weights $w$. For any two augmented data $x, x' \in \mathcal{X}$, $w_{x x'}$ is the marginal probability of generating the pair $(x,x')$:
\begin{align}
\begin{split}
w_{x x^{\prime}} &\triangleq \alpha \sum_{i \in \mathcal{Y}_l}\mathbb{E}_{\bar{x}_{l} \sim {\mathcal{P}_{l_i}}} \mathbb{E}_{\bar{x}'_{l} \sim {\mathcal{P}_{l_i}}} \tikzmarknode{c2}{\highlight{red}{$\mathcal{T}(x | \bar{x}_{l}) \mathcal{T}\left(x' | \bar{x}'_{l}\right)  $}} \\ &+ 
    \beta \mathbb{E}_{\bar{x}_{u} \sim {\mathcal{P}_u}} \tikzmarknode{c1}{\highlight{blue}{$ \mathcal{T}(x| \bar{x}_{u}) \mathcal{T}\left(x'| \bar{x}_{u}\right) $}},
    \label{eq:def_wxx}
    \vspace{1cm}
\end{split}
\end{align}
\begin{tikzpicture}[overlay,remember picture,>=stealth,nodes={align=left,inner ysep=1pt},<-]
    \path (c2.south) ++ (0,0.1em) node[anchor=north west,color=red!67] (c2t){\textit{ case (b)}};
    \draw [color=red!87](c2.south) |- ([xshift=-0.3ex,color=red] c2t.south east);
    \path (c1.south) ++ (0,0.1em) node[anchor=north west,color=blue!67] (c1t){\textit{ case (a)}};
    \draw [color=blue!87](c1.south) |- ([xshift=-0.3ex,color=blue]c1t.south east);
\end{tikzpicture}

where $\alpha,\beta$ modulates the importance between unlabeled and labeled data. 
The magnitude of $w_{xx'}$ indicates the ``positiveness'' or similarity between  $x$ and $x'$. 
We then use $w_x = \sum_{x' \in \mathcal{X}}w_{xx'}$ to denote the total edge weights connected to vertex $x$.

As a standard technique in graph theory~\cite{chung1997spectral}, we use the \textit{normalized adjacency matrix}:
\begin{equation}
    \dot{A}\triangleq D^{-1 / 2} A D^{-1 / 2},
    \label{eq:def}
\end{equation}
where  $A \in \mathbb{R}^{N \times N}$ is adjacency matrix with entries $A_{x x^\prime}=w_{x x^{\prime}}$ and $D \in \mathbb{R}^{N \times N}$ is a diagonal matrix with $D_{x x}=w_x.$ The normalization balances the degree of each node,  reducing the influence of vertices with very large degrees. The adjacency matrix defines the probability of $x$ and $x^{\prime}$  being considered as the positive pair from the perspective of augmentation, which helps derive the NCD Spectral Contrastive Loss as we show next.

\subsection{NCD Spectral Contrastive Learning}
\label{sec:ncd-scl}
In this subsection, we propose a formal definition of NCD Spectral Contrastive Loss, which can be derived from a spectral decomposition of $\dot{A}$. The derivation of the loss is inspired by ~\cite{haochen2021provable}, and allows us to theoretically show the equivalence between learning  feature embeddings and the projection on the top-$k$ SVD components of $\dot{A}$. Importantly, such equivalence facilitates the theoretical understanding based on the semantic relation between known and novel classes encoded in $\dot{A}$. 

Specifically, we consider low-rank matrix approximation:
\begin{equation}
    \min _{F \in \mathbb{R}^{N \times k}} \mathcal{L}_{\mathrm{mf}}(F, A)\triangleq\left\|\Dot{A}-F F^{\top}\right\|_F^2
    \label{eq:lmf}
\end{equation}
According to the Eckart–Young–Mirsky theorem~\cite{eckart1936approximation}, the minimizer of this loss function is $F^*\in \mathbb{R}^{N \times k}$ such that $F^* F^{*\top}$ contains the top-$k$ components of $\Dot{A}$'s SVD decomposition. 

Now, if we view each row $\*f_x^{\top}$ of $F$ as a learned feature embedding  $f:\mathcal{X}\mapsto \mathbb{R}^k$, the $\mathcal{L}_{\mathrm{mf}}(F, A)$ can be written as a form of the contrastive learning objective. We formalize this connection in Theorem~\ref{th:ncd-scl} below.

\begin{theorem}
\label{th:ncd-scl} 
We define $\*f_x = \sqrt{w_x}f(x)$ for some function $f$. Recall $\alpha,\beta$ are hyper-parameters defined in Eq.~\eqref{eq:def_wxx}. Then minimizing the loss function $\mathcal{L}_{\mathrm{mf}}(F, A)$ is equivalent to minimizing the following loss function for $f$, which we term \textbf{NCD Spectral Contrastive Loss (NSCL)}:
\begin{align}
\begin{split}
    \mathcal{L}_{nscl}(f) &\triangleq - 2\alpha \mathcal{L}_1(f) 
- 2\beta  \mathcal{L}_2(f) \\ & + \alpha^2 \mathcal{L}_3(f) + 2\alpha \beta \mathcal{L}_4(f) +  
\beta^2 \mathcal{L}_5(f),
\label{eq:def_nscl}
\end{split}
\end{align} where
\begin{align*}
    \mathcal{L}_1(f) &= \sum_{i \in \mathcal{Y}_l}\underset{\substack{\bar{x}_{l} \sim \mathcal{P}_{{l_i}}, \bar{x}'_{l} \sim \mathcal{P}_{{l_i}},\\x \sim \mathcal{T}(\cdot|\bar{x}_{l}), x^{+} \sim \mathcal{T}(\cdot|\bar{x}'_l)}}{\mathbb{E}}\left[f(x)^{\top} {f}\left(x^{+}\right)\right] , \\
    \mathcal{L}_2(f) &= \underset{\substack{\bar{x}_{u} \sim \mathcal{P}_{u},\\x \sim \mathcal{T}(\cdot|\bar{x}_{u}), x^{+} \sim \mathcal{T}(\cdot|\bar{x}_u)}}{\mathbb{E}}
\left[f(x)^{\top} {f}\left(x^{+}\right)\right], \\
    \mathcal{L}_3(f) &= \sum_{i \in \mathcal{Y}_l}\sum_{j \in \mathcal{Y}_l}\underset{\substack{\bar{x}_l \sim \mathcal{P}_{{l_i}}, \bar{x}'_l \sim \mathcal{P}_{{l_{j}}},\\x \sim \mathcal{T}(\cdot|\bar{x}_l), x^{-} \sim \mathcal{T}(\cdot|\bar{x}'_l)}}{\mathbb{E}}
\left[\left(f(x)^{\top} {f}\left(x^{-}\right)\right)^2\right], \\
    \mathcal{L}_4(f) &= \sum_{i \in \mathcal{Y}_l}\underset{\substack{\bar{x}_l \sim \mathcal{P}_{{l_i}}, \bar{x}_u \sim \mathcal{P}_{u},\\x \sim \mathcal{T}(\cdot|\bar{x}_l), x^{-} \sim \mathcal{T}(\cdot|\bar{x}_u)}}{\mathbb{E}}
\left[\left(f(x)^{\top} {f}\left(x^{-}\right)\right)^2\right], \\
    \mathcal{L}_5(f) &= \underset{\substack{\bar{x}_u \sim \mathcal{P}_{u}, \bar{x}'_u \sim \mathcal{P}_{u},\\x \sim \mathcal{T}(\cdot|\bar{x}_u), x^{-} \sim \mathcal{T}(\cdot|\bar{x}'_u)}}{\mathbb{E}}
\left[\left(f(x)^{\top} {f}\left(x^{-}\right)\right)^2\right].
\end{align*}
\end{theorem}
\begin{proof} (\textit{sketch})
We can expand $\mathcal{L}_{\mathrm{mf}}(F, A)$ and obtain
\begin{align*}
&\mathcal{L}_{\mathrm{mf}}(F, A) =\sum_{x, x^{\prime} \in \mathcal{X}}\left(\frac{w_{x x^{\prime}}}{\sqrt{w_x w_{x^{\prime}}}}-\*f_x^{\top} \*f_{x^{\prime}}\right)^2 = const + \\
&\sum_{x, x^{\prime} \in \mathcal{X}}\left(-2 w_{x x^{\prime}} f(x)^{\top} {f}\left(x^{\prime}\right)+w_x w_{x^{\prime}}\left(f(x)^{\top}{f}\left(x^{\prime}\right)\right)^2\right)
\end{align*} 
The form of $\mathcal{L}_{nscl}(f)$ is derived from plugging $w_{xx'}$ (defined in Eq.~\eqref{eq:def_wxx}) and $w_x$. 
We include the details in Appendix~\ref{sec:proof-nscl}. 
\end{proof}

\textbf{Interpretation of $\mathcal{L}_{nscl}(f)$.} 
At a high level, $\mathcal{L}_1$ and $\mathcal{L}_2$ push the embeddings of \textbf{positive pairs} to be closer while $\mathcal{L}_3$, $\mathcal{L}_4$ 
 and $\mathcal{L}_5$ pull away the embeddings of \textbf{negative pairs}. In particular, $\mathcal{L}_1$ samples two random augmentation views of two images from labeled data with the \textbf{same} class label, and $\mathcal{L}_2$ samples two views from the same image in $\mathcal{X}_{u}$. For negative pairs, $\mathcal{L}_3$ uses two augmentation views from two samples in $\mathcal{X}_{l}$ with \textbf{any} class label. $\mathcal{L}_4$ uses two views of one sample in $\mathcal{X}_{l}$ and another one in $\mathcal{X}_{u}$. $\mathcal{L}_5$ uses two views from two random samples in $\mathcal{X}_{u}$. 

\section{Theoretical Analysis}
\label{sec:theory}
So far we have presented a spectral approach for NCD based on the augmentation graph. Under this formulation, we now formally investigate and analyze:
\emph{\textbf{when and how does the known class help discover novel class?}}  We start by showing that analyzing the linear probing performance is equivalent to analyzing the regression residual using singular vectors of $\Dot{A}$ in Sec.~\ref{sec:setup}. We then construct a toy example to illustrate and verify the key insight in Sec.~\ref{sec:toy}. We finally provide a formal theory for the general case in Sec.~\ref{sec:theory_main}.

\subsection{Theoretical Setup}
\label{sec:theory_setup}
\textbf{Representation for unlabeled data.} We apply NCD spectral learning objective $\mathcal{L}_{nscl}(f)$ in Equation~\ref{eq:def_nscl} and assume the optimizer is capable to obtain the representation that minimizes the loss. We can then obtain the $F^*$ s.t. $F^*F^{*\top}$ are the top-$k$ components of $\Dot{A}$’s SVD decomposition. 
To ease the analysis, we will focus on the top-$k$ singular vectors $V^* \in \mathbb{R}^{N\times k}$ of $\Dot{A}$ such that $F^* = V^* \sqrt{\Sigma_k}$, where $\Sigma_k$ is the diagonal matrix with top-$k$ singular values ($\sigma_1, ..., \sigma_k$). 

Since we are primarily interested in the unlabeled data, we split $V^*$ into two parts: $U^* \in \mathbb{R}^{N_u\times k}$ for unlabeled data and $L^* \in \mathbb{R}^{N_l\times k}$ for labeled data, respectively. Assuming the first $N_l$ rows/columns in $\Dot{A}$ corresponds to the labeled data, we can conveniently rewrite $V^*$ as: 
\begin{equation}
    V^* = \left[\begin{array}{c}
         L^* (\text{labeled part})\\
         U^* (\text{unlabeled part})
    \end{array}\right]
\end{equation}

\textbf{Linear probing evaluation.} With the learned representation for the unlabeled data, we can evaluate NCD quality by the linear probing performance. The strategy is commonly used in self-supervised learning~\cite{chen2020simclr}. Specifically, the weight of a linear classifier is denoted as $\*M \in \mathbb{R}^{k \times |\mathcal{Y}_u|}$.
The class prediction is given by $h(x;f, \*M) = \operatorname{argmax}_{i \in \mathcal{Y}_u} (f(x)^\top \*M)_i$. The linear probing performance is given by the least error of all possible linear classifiers:
\begin{equation}
\mathcal{E}(f)\triangleq\underset{{\*M}\in \mathbb{R}^{k \times |\mathcal{Y}_u|}}{\operatorname{min}}  \underset{{x \in \mathcal{X}_u}}{\sum} \mathbbm{1}\left[y(x) \neq h(x;f, \*M)\right],
\label{eq:def_error}
\end{equation}

where $y(x)$ indicates the ground-truth class of $x$.

\begin{figure*}[t]
    \centering
    \includegraphics[width=0.85\linewidth]{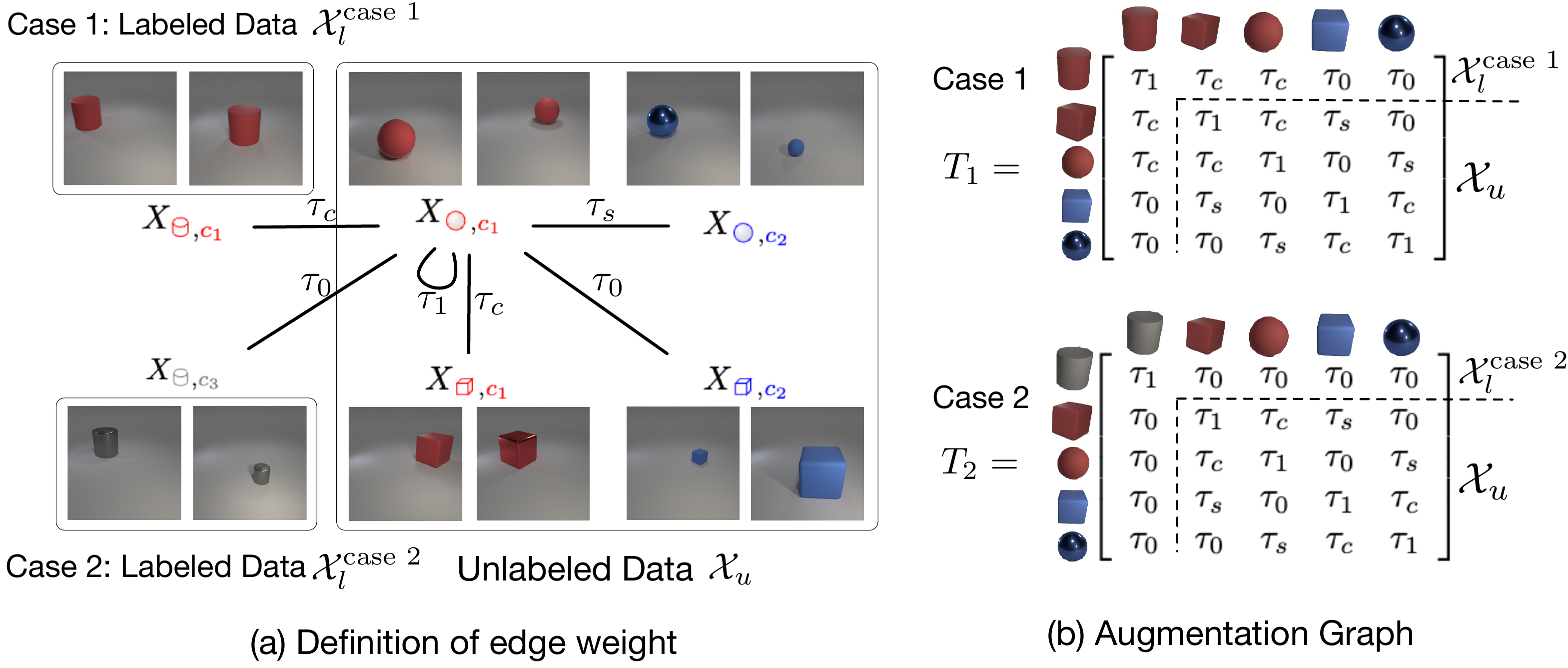}
    \caption{An illustrative example for theoretical analysis. \textbf{(a)} The unlabeled data $\mathcal{X}_u$ consists of 3D objects of sphere/cube with red/blue colors. We consider two cases of labeled data: (1) Case 1 uses a red cylinder $X_{\textcolor{red}{\cylinder{0.4}}, \textcolor{red}{c_1}}$ which is correlated with the target novel class (red). (2) Case 2 uses gray cylinder $X_{\textcolor{gray}{\cylinder{0.4}}, \textcolor{gray}{c_3}}$ which has no correlation with $\mathcal{X}_u$. \textbf{(b)} The augmentation matrices for case 1 and case 2 respectively. See definition in Eq.~\eqref{eq:def_edge}. Best viewed in color.}
    \label{fig:toy_setting}
\end{figure*}

\textbf{Residual analysis.} With defined $U^*$, we can bound the linear probing error $\mathcal{E}(f)$  by the residual of the regression error $\mathcal{R}(U^*)$ as we show in Lemma~\ref{lemma:cls_bound} with proof in Appendix~\ref{sec:sup_cls_bound}. 
\begin{lemma}
Denote the $\mathbf{y}(x) \in \mathbb{R}^{|\mathcal{Y}_u|}$ as a one-hot vector whose $y(x)$-th position is 1 and 0 elsewhere. Let
$\mathbf{Y} \in \mathbb{R}^{N_u \times |\mathcal{Y}_u|}$ as a binary mask whose rows are stacked by $\mathbf{y}(x)$. We have: 
    $$\mathcal{R}(U^*) \triangleq \underset{{\*M}\in \mathbb{R}^{k \times |\mathcal{Y}_u|}}{\operatorname{min}} \|\mathbf{Y} - U^* \*M \|^2_F \geq \frac{1}{2}\mathcal{E}(f).$$
\label{lemma:cls_bound}
\end{lemma}
Note that we can rewrite $\mathcal{R}(U^*)$ as the summation of individual residual terms $\mathcal{R}(U^*, \Vec{y}_i)$: 
$
    \mathcal{R}(U^*) = \sum_{i \in \mathcal{Y}_u} \mathcal{R}(U^*, \Vec{y}_i),
$
where
$$ \mathcal{R}(U^*, \Vec{y}_i) \triangleq  \underset{{\Vec{\mu}_i}\in \mathbb{R}^{k}}{\operatorname{min}} \|\Vec{y}_i - U^* \Vec{\mu}_i \|^2_2, $$
and $\Vec{y}_i \in \mathbb{R}^{N_u}$ is the $i$-th column of $\mathbf{Y}$ and $\Vec{\mu}_i \in \mathbb{R}^{k}$ is the $i$-th column of $\*M$. Without losing the generality, our analysis will revolve around the residual term $\mathcal{R}(U^*, \Vec{y}_i)$ for specific class $i$. It is clear that if learned representation $U^*$ encodes more information of the label vector $\Vec{y}_i$, the residual $\mathcal{R}(U^*, \Vec{y}_i)$ becomes smaller\footnote{In an extreme case, if the first column of $U^*$ is exactly the same as $\Vec{y}_i$, one can set $\Vec{\mu}_i = [1, 0 , 0, ...]^{\top}$ to make residual zero.}. Such insight can be used to investigate which type of known class is more helpful for learning the representation of novel classes.

\subsection{An Illustrative Example}
\label{sec:toy}

We consider a toy example that helps illustrate the core idea of our theoretical findings. Specifically, the example aims to cluster 3D objects of different colors and shapes, as shown in Figure~\ref{fig:toy_setting} (a). These images are generated by a 3D rendering software~\cite{johnson2017clevr} with user-defined properties including colors, shape, size, position, etc. 

In what follows, we define two data configurations and corresponding graphs, where the labeled data is correlated  with the attribute of unlabeled data (\textbf{case 1}) vs. not (\textbf{case 2}). We are interested in contrasting the representations (in form of singular vectors) and residuals  derived from both scenarios. The proof of all theorems in this section is provided in Appendix~\ref{sec:proof-eigen}.

\textbf{Motivation and data design.}  For simplicity, we focus on two main properties: color and shape. Formally, the images with shape $s$ and color $c$ are sampled from a generation procedure $\mathcal{G}$:  $$X_{s, c} \sim \mathcal{G}(s, c),$$ 
where $s \in \{\cube{1} (\text{cube}), \sphere{0.7}{gray} (\text{sphere}),     \cylinder{0.6} (\text{cylinder})\}$ and  $c \in \{\textcolor{red}{c_1} (\text{red}), \textcolor{blue}{c_2} (\text{blue}), \textcolor{gray}{c_3} (\text{gray})\}$. We then construct our unlabeled dataset containing red/blue cubes/spheres as: 

$$\mathcal{X}_u \triangleq \{X_{\textcolor{red}{\cube{0.6}}, \textcolor{red}{c_1}}, X_{\textcolor{red}{\sphere{0.5}{red}}, \textcolor{red}{c_1}}, X_{\textcolor{blue}{\cube{0.6}}, \textcolor{blue}{c_2}}, X_{\textcolor{blue}{\sphere{0.5}{blue}}, \textcolor{blue}{c_2}}\}.$$

For simplicity, we assume each element in $\mathcal{X}_u$ is a single example. W.o.l.g, we also assume the red cube and red sphere form the target novel class. Then the corresponding labeling vector on $\mathcal{X}_u$ is defined by: $$\Vec{y} \triangleq \{1,1,0,0\}.$$ 
To answer \textit{``when and how does the known class help discover novel class?''}, we  construct two separate scenarios: one helps and the other one does not. Specifically, in the first case, we let the labeled data $\mathcal{X}_{l}^{\text{case 1}}$ be strongly correlated with the target class (red color) in unlabeled data:
$$\mathcal{X}_{l}^{\text{case 1}} \triangleq \{X_{\textcolor{red}{\cylinder{0.4}}, \textcolor{red}{c_1}}\} (\text{red cylinder}).$$  
In the second case, we construct the labeled data that has no correlation with any novel classes. We use gray cylinders which have no overlap in either shape and color: 
$$\mathcal{X}_{l}^{\text{case 2}} \triangleq \{X_{\textcolor{gray}{\cylinder{0.4}}, \textcolor{gray}{c_3}}\}  (\text{gray cylinder}).$$ 
Putting it together, our entire training dataset is 
$\mathcal{X}^{\text{case 1}} = \mathcal{X}_{l}^{\text{case 1}} \cup \mathcal{X}_{u}$ or  $\mathcal{X}^{\text{case 2}} = \mathcal{X}_{l}^\text{case 2} \cup \mathcal{X}_{u}$. We aim to verify the hypothesis that: 
\textit{the representation learned by $\mathcal{X}^{\text{case 1}}$ provides a much smaller regression residual to $\Vec{y}$ than $\mathcal{X}^{\text{case 2}}$ for color class. }

\textbf{Augmentation graph.} 
Based on the data, we now define the probability of augmenting an image $X_{s, c}$ to another $X'_{s', c'}$:
\begin{align}
    \mathcal{T}\left(X'_{s', c'} \mid X_{s, c} \right)=\left\{\begin{array}{lll}
    \tau_1 & \text { if } & s=s', c=c', \\
    \tau_{s} & \text { if } & s=s', c \neq c', \\
    \tau_{c} & \text { if } & s\neq s', c=c', \\
    \tau_0 & \text { if } & s\neq s', c\neq c', \\
    \end{array}\right. 
    \label{eq:def_edge}
\end{align}
It is natural to assume the magnitude order that follows $\tau_1 \gg \max(\tau_{s},\tau_{c})$ and $\min(\tau_{s},\tau_{c}) \gg \tau_0$.  In two data settings $\mathcal{X}^\text{case 1}$ and $\mathcal{X}^\text{case 2}$, 
the corresponding augmentation matrices ${T}_1, {T}_2$ formed by $\mathcal{T}\left(\cdot|\cdot\right)$ are presented in Fig.~\ref{fig:toy_setting} (b). 
According to Eq.~\eqref{eq:def_wxx}, it can be verified that the adjacency matrices are $A_1 = T_1^2$ and $A_2 = T_2^2$ respectively. 

\textbf{Main analysis.}  
We are primarily interested in analyzing the difference of the representation space derived from $A_1$ vs. $A_2$. Since $\tau_1 \gg \max(\tau_{s},\tau_{c})$, one can show that $A_1$ and $A_2$ are positive-definite. The singular vector is thus equivalent to the eigenvector. Also note that $A_1$ and their square root $T_1$ have the same eigenvectors and order. It is thus equivalent to analyzing the eigenvectors of $T_1$. Same with $A_2$ and $T_2$. In this toy example, we consider the eigenvalue problem of the unnormalized adjacency matrix\footnote{The normalized/unnormalized adjacency matrix corresponds to the NCut/RatioCut problem respectively~\cite{von2007tutorial}.} for simplicity. 

We put analysis on the top-$2$ eigenvectors $V^*_{1},V^*_{2}  \in \mathbb{R}^{5\times 2}$ for  $A_1$/$A_2$ ---- as we will see later, the top-$1$ eigenvector of $T_1/T_2$ usually functions at distinguishing known vs novel data, while the 2nd eigenvector functions at distinguishing color or shape. 

We let $U^*_1 \in \mathbb{R}^{4\times 2}$ contains the last 4 rows of $V^*_{1}$, and corresponds to the ``representation''  for the unlabeled data only. $U^*_2$ is defined in the same way \emph{w.r.t.} $A_2$. We have the following theorem:

\begin{theorem}
   Assume $\tau_1 = 1$, $\tau_0 = 0$, $\tau_s < 1.5\tau_c$. We have
$$
U^*_1=\left[\begin{array}{ccccc}
 a_1 & a_1 & b_1 & b_1 \\
 a_2 & a_2 & b_2 & b_2 \\
\end{array}\right]^{\top},
$$
where $a_1,b_1$ are some positive real numbers, and $a_2,b_2$ has different signs.  
$$U^*_2=\left\{\begin{array}{ll}   
\frac{1}{2}\left[\begin{array}{cccc}
1 & 1 & 1 & 1\\ 1 & 1 & -1 & -1\\
\end{array}\right]^{\top}, & \text{if } \tau_{s} < \tau_{c}, \\
\frac{1}{2}\left[\begin{array}{cccc}
1 & 1 & 1 & 1\\ -1 & 1 & -1 & 1\\
\end{array}\right]^{\top}, & \text{if } \tau_{s} > \tau_{c}, 
\end{array}\right. $$
With label vector $\Vec{y} = \{1,1,0,0\}$, we have 
\begin{equation}
    \mathcal{R}(U^*_1, \Vec{y}) = 0, \mathcal{R}(U^*_2, \Vec{y}) = \left\{\begin{array}{ll}    
    0, & \text{if } \tau_{s} < \tau_{c}\\
     1, &  \text{if } \tau_{s} > \tau_{c}.
    \end{array}\right. 
\end{equation}
    \label{th:toy_extreme}
\end{theorem}
\textbf{Interpretation of Theorem~\ref{th:toy_extreme}:} 
The discussion can be divided into two cases: (1) 
$\tau_{s} < \tau_{c}$. (2) $\tau_{s} > \tau_{c}$. In the \textbf{first case} $\tau_{s} < \tau_{c}$, the connection between the same-color data pair is already stronger than the same-shape data pair. Thus the eigenvector corresponding to color information ($\frac{1}{2}[1,1,-1,-1]^\top$)  will be more prominent (and ranked higher in $U^*_2$) than ``shape eigenvector'' ($\frac{1}{2}[-1,1,-1,1]^\top$). 
Since the feature $U^*_2$ already encodes sufficient information (color) of the labeling vector $\Vec{y}$, fitting $\Vec{y}$ becomes easy and the residual $\mathcal{R}(U^*_2,\Vec{y})$ becomes 0.

In NCD, \textbf{we are more interested in the second case} ($\tau_{s} > \tau_{c}$), where unlabeled data indeed need some help from labeled data for better clustering. Such help comes from the semantic connection between labeled data and unlabeled data. In our toy example, the semantic connection comes from the first row/column of ${T}_1$ and ${T}_2$. However, the first row/column of ${T}_2$ is $[1,0,0,0,0]$, which means there is no extra information offered from $\mathcal{X}_{l}^\text{case 2}$. It is because $\mathcal{X}_{l}^\text{case 2}$ contains gray cylinders which have neither colors nor shapes connection to unlabeled data $\mathcal{X}_u$. Contrarily, $\mathcal{X}_{l}^\text{case 1}$ with red cylinder provides strong color prior. This allows the  ``color eigenvector'' ($[a_2, a_2, -b_2, -b_2]$) to become a main component in $U^*_1$, making the residual $\mathcal{R}(U^*_1,\Vec{y})=0$ even when $\tau_{s} > \tau_{c}$. 

\textbf{Main takeaway.} In Theorem~\ref{th:toy_extreme}, we have verified the hypothesis that incorporating labeled data $\mathcal{X}_{l}^\text{case 1}$ (red  cylinder) can reduce the residual $\mathcal{R}(U^*_1,\Vec{y})$ more than using  $\mathcal{X}_{l}^\text{case 2}$, especially when color is a weaker signal than shape in unlabeled data. 

\textbf{Extension: A more general result.} 
Note that $T_1$ and $T_2$ are special cases of the following $T(t)$ with $t \in [\tau_0, \tau_c]$:

\begin{equation*}
T(t)=\left[\begin{array}{ccccc}
\tau_1 & t & t & \tau_0 & \tau_0 \\
t & \tau_1 & \tau_{c} & \tau_{s} & \tau_0  \\
t & \tau_{c} & \tau_1 & \tau_0 & \tau_{s}  \\
\tau_0 & \tau_{s} & \tau_0 & \tau_1 & \tau_{c}  \\
\tau_0 & \tau_0 & \tau_{s} & \tau_{c} & \tau_1  \\
\end{array}\right],
\end{equation*}
where $t$ indicates the strength of the connection between labeled data and a novel class in unlabeled data. Let $U^*_t$ be the representation for unlabeled data derived from $T(t)$.  The following theorem indicates that the residual decreases when $t$ increases and the residual becomes 0 when $t$ is larger than a threshold $\Bar{t}$ depending on the gap between $\tau_s$ and $\tau_c$. 
\begin{theorem}
     Assume  $\tau_1 = 1$, $\tau_0 = 0$, $1.5\tau_c > \tau_s > \tau_c$. Let $\Bar{t} = \sqrt{\frac{2(\tau_s-\tau_c)^2\tau_c}{2\tau_c - \tau_s}}$, $r: \mathbb{R} \mapsto (0,1) $ be a real value function, we have 
     \begin{equation}
         \mathcal{R}(U^*_t, \Vec{y}) = \left\{\begin{array}{ll}    
     0, &  \text{if } t \in (\Bar{t}, \tau_s), \\
    r(t), & \text{if } t \in (0, \Bar{t}), \\ 
     1, &  \text{if } t = 0. \end{array}\right. 
     \end{equation}
    \label{th:toy_general}
\end{theorem}

\textbf{Can adding labeled data be harmful?} We exemplify the scenario in Figure~\ref{fig:teaser}, where the umbrella images are
given as a known class, undesirably causing the “mushroom with
umbrella shape” to be grouped together. To formally analyze this case,  we construct \textbf{case 3}:
$$\mathcal{X}_{l}^{\text{case 3}} \triangleq \{X_{\textcolor{gray}{\cube{0.5}}, \textcolor{gray}{c_3}}\}  (\text{gray cube}).$$  
In this case, we have the following Lemma~\ref{th:toy_harmful}. 
\begin{lemma}
    If  $\frac{\tau_c}{\tau_s} \in (1, 1.5)$,  
$
    \mathcal{R}(U^*_3, \Vec{y})- \mathcal{R}(U^*_2, \Vec{y})=1. 
$
    \label{th:toy_harmful}
\end{lemma}
The residual in case 3 is now larger than in case 2, since the shape is treated as a more important feature than the color feature (which relates to the target class). The main takeaway of this lemma is that the labeled data can be harmful when its connection with unlabeled data is undesirably stronger in the spurious feature dimension.

\begin{figure}[t]
    \centering
    \includegraphics[width=0.9\linewidth]{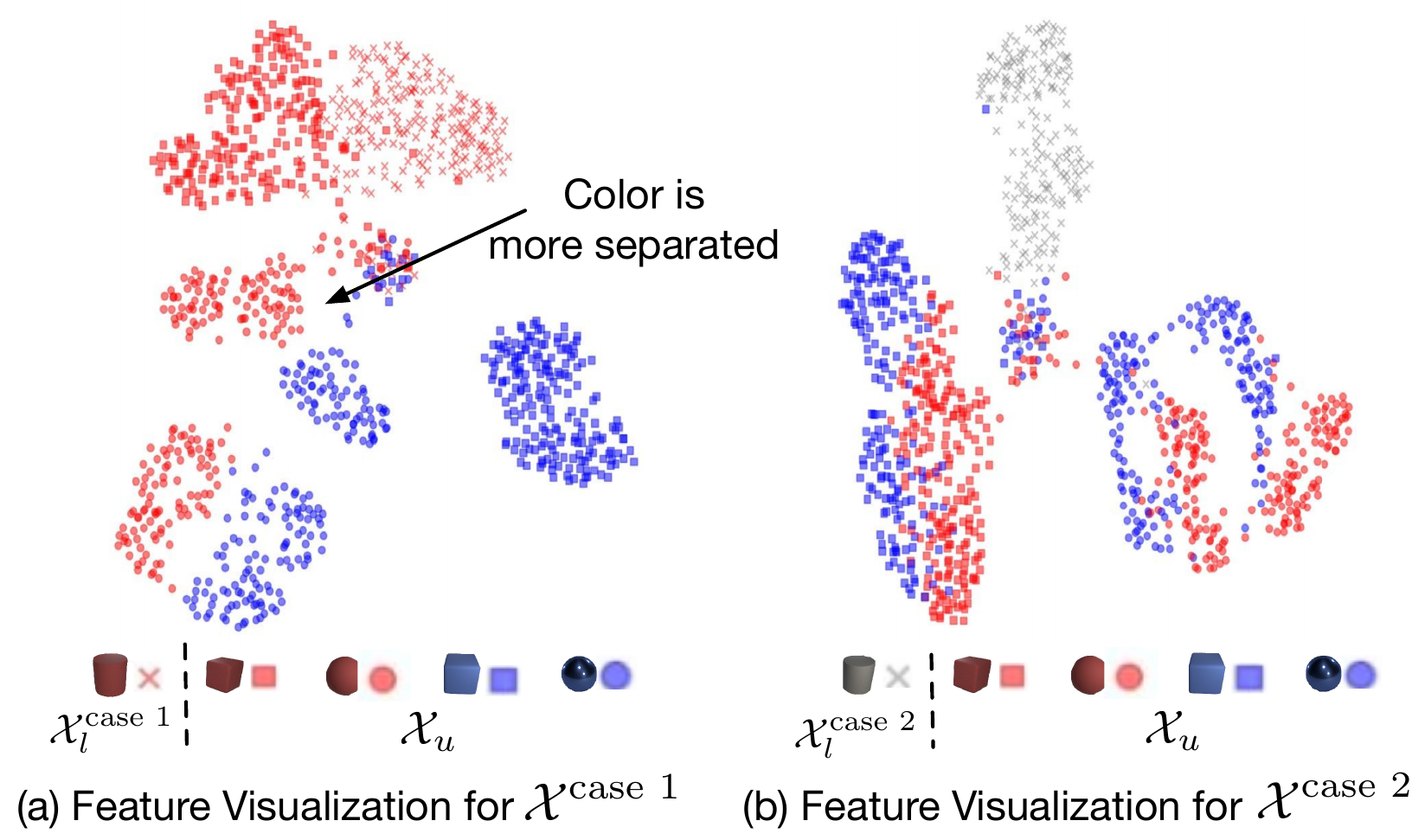}
    \caption{UMAP~\citep{umap} visualization of the feature embedding learned from $\mathcal{X}^\text{case 1}$ and $\mathcal{X}^\text{case 2}$ respectively. The model is trained with NCD Spectral Contrastive Loss. }
    \label{fig:toy_vis}
\end{figure}
\textbf{Qualitative results.} The theoretical results can be verified in our empirical results by visualization in Fig.~\ref{fig:toy_vis}. Due to the space limitation, we include experimental details in Appendix~\ref{sec:sup_exp_vis}. As seen in Fig.~\ref{fig:toy_vis} (a), the features of unlabeled data $\mathcal{X}_u$ jointly learned with red cylinder $\mathcal{X}_{l}^\text{case 1}$ are more distinguishable by color attribute,  as opposed to Fig.~\ref{fig:toy_vis} (b). 

\subsection{Main Theory}
\label{sec:theory_main}

The toy example offers an important insight that using the labeled data help reduce the residual when it provides the missing information of unlabeled data. In this section, we will formalize this insight by extending the toy example to a more general setting with $N$ samples. We start with the definition of notations. 

\textbf{Notations.} Recall that $V^* \in \mathbb{R}^{N \times k}$ is defined as the top-$k$ singular vectors of $\Dot{A}$, which is further split into two parts $L^* = \left[l_1, l_2, \cdots, l_k\right] \in \mathbb{R}^{N_l \times k}$, $U^* = \left[u_1, u_2, \cdots, u_k\right] \in \mathbb{R}^{N_u \times k}$, for labeled and unlabeled samples respectively. Then we let $V^{\flat} \in \mathbb{R}^{N \times (N-k)}$ be the remaining singular vectors of $\Dot{A}$ except top-$k$. Similarly, we split $V^{\flat}$ into two parts ($L^{\flat} = \left[l_{k+1}, l_{k+2}, \cdots, l_N\right] \in \mathbb{R}^{N_l \times (N-k)}$, $U^{\flat} = \left[u_{k+1}, u_{k+2}, \cdots, u_N\right] \in \mathbb{R}^{N_u \times (N-k)}$). 

We now present our first main result in Theorem~\ref{th:no_approx}.
\begin{theorem}
    Denote the projection matrix $\mathsf{P}_{L^{\flat}} = L^{\flat\top}(L^{\flat}L^{\flat\top})^{\dag}L^{\flat}$, where $^{\dag}$ denotes the Moore-Penrose inverse. For any labeling vector $\Vec{y} \in \{0,1\}^{N_u}$, we have
    \begin{equation}
        \mathcal{R}(U^*, \Vec{y}) \leq \|(I-\mathsf{P}_{L^{\flat}}) U^{\flat\top} \Vec{y}\|^2_2.
        \label{eq:R_bound_no_approx}
    \end{equation} 
    \label{th:no_approx}
\end{theorem}
\textbf{Interpretation of Theorem~\ref{th:no_approx}.} The bound of residual in Ineq.~\eqref{eq:R_bound_no_approx} is composed of two projections:  $U^{\flat\top}$ and $(I-\mathsf{P}_{L^{\flat}})$. 
We first consider the ignorance space formed by the first projection: 
$$\text{\textbf{ignorance space}} \triangleq U^{\flat\top} \Vec{y},$$ 
 which contains the information of the labeling vector $\Vec{y}$ that is not encoded in the learned representation $U^{*}$ of the unlabeled data. Intuitively, when $\mathcal{R}\left(U^*, \vec{y}\right)>0$, the labeling vector $\vec{y}$ does not lie in the span of the existing representation $U^*$. On the other hand, $\mathcal{R}\left([\begin{array}{cc}  U^* & U^{\flat}\end{array}], \vec{y}\right)=0$ since $U^*$ together with $U^{\flat}$ forms a full rank space. We also define a measure of the ignorance degree of the current feature space: 
 $\text{\textbf{ignorance degree}} \triangleq \mathfrak{T}(\Vec{y}) = \frac{\|U^{\flat\top} \Vec{y}\|_2}{\|\Vec{y}\|_2}.$ 

The second projection matrix $(I-\mathsf{P}_{L^{\flat}})$ is composed of $L^{\flat}$, which we deem as the  extra knowledge from known classes: 
\begin{equation*}
\text{\textbf{extra knowledge}} \triangleq L^{\flat}.
\end{equation*} Multiplying the second projection matrix $(I-\mathsf{P}_{L^{\flat}})$ further reduces the norm of the ignorance space by considering the extra knowledge from labeled data, since $\mathsf{P}_{L^{\flat}}$ is a projection matrix that projects a vector to the linear span of $L^{\flat}$. In the extreme case, when $U^{\flat\top} \Vec{y}$ fully lies in the linear span of $L^{\flat}$, the residual $\mathcal{R}(U^*, \Vec{y})$ goes 0.

Next, we present another main theorem that bounds the linear probing error  $\mathcal{E}(f)$ based on the relations between the known and novel classes. See Appendix~\ref{sec:sup_main_proof} for a detailed discussion and assumption.
\begin{theorem} 
Let $[\begin{array}{cc} A_{ul}\in \mathbb{R}^{N_u \times N_l}, A_{uu}\in \mathbb{R}^{N_u \times N_u} 
\end{array}]$ be the sub-matrix of the last $N_u$ rows of $\Dot{A}$, and $q_i$ be the $i$-th eigenvector of $A_{uu}$. 
The linear probing error can be bounded as follows:
    \begin{equation*}
        \mathcal{E}(f) \lesssim \frac{2N_u}{|\mathcal{Y}_u|}\left(\sum_i^{|\mathcal{Y}_u|} \tikzmarknode{igdegree}{\highlight{red}{$\mathfrak{T}(\Vec{y}_i)$}}(1-\tikzmarknode{kappa}{\highlight{blue}{$\kappa(\Vec{y}_i)^2$}}) + \frac{\|\Dot{A} - \bar{A}\|_2}{\sigma_{k} - \sigma_{k+1}} \right), 
    \end{equation*}
\begin{tikzpicture}[overlay,remember picture,>=stealth,nodes={align=left,inner ysep=1pt},<-]
    \path (igdegree.north) ++ (0,0.5em) node[anchor=south west,color=red!67] (ig_arrow){\textit{ ignorance degree }};
    \draw [color=red!87](igdegree.north) |- ([xshift=-0.3ex,color=red] ig_arrow.south east);
    \path (kappa.south) ++ (0,-0.7em) node[anchor=north west,color=blue!67] (kappa_arrow){\textit{ knowledge coverage}};
    \draw [color=blue!87](kappa.south) |- ([xshift=-0.3ex,color=blue]kappa_arrow.south east);
\end{tikzpicture}
    where     
    $$\kappa(\Vec{y}) = \cos(\Bar{U}^{\flat\top} \Vec{y},\Bar{\mathfrak{l}}^{\flat}) 
    \gtrsim \min_{i > k, j > k} \frac{2\sqrt{\frac{\Vec{y}^{\top}q_i}{\Vec{\eta}_u^{\top}q_i}\frac{\Vec{y}^{\top}q_j}{\Vec{\eta}_u^{\top}q_j}}}{\frac{\Vec{y}^{\top}q_i}{\Vec{\eta}_u^{\top}q_i}+\frac{\Vec{y}^{\top}q_j}{\Vec{\eta}_u^{\top}q_j}},$$
    \label{th:main}
\end{theorem}
and $\Bar{A}$ is the approximation of $\Dot{A}$ by taking the expectation in the rows/columns of labeled samples (Appendix~\ref{sec:sup_with_approx}) with a similar motivation as the SBM model~\cite{holland1983stochastic}. In such condition,  $\Bar{U}^{\flat\top}$, $\Bar{\mathfrak{l}}^{\flat}$ and $\eta_u$ is the approximation to $U^{\flat\top}$,  $L^{\flat}$ and $A_{ul}$ accordingly.

\textbf{Interpretation of $\kappa(\Vec{y})$.} We provide the detailed derivation of $\kappa(\Vec{y})$ in Lemma~\ref{lemma:sup_bound_kappa}. Intuitively, $\kappa(\Vec{y})$ measures the usefulness and relevance of knowledge from known classes for NCD. We formally call it coverage, which measures the cosine distance between the ignorance space and the extra knowledge: 
\begin{equation*}
    \text{\textbf{coverage}} \triangleq \kappa(\Vec{y}) = \cos(\tikzmarknode{ignorance}{\highlight{red}{$ \Bar{U}^{\flat\top} \Vec{y}$}}, \tikzmarknode{knowledge}{\highlight{blue}{$\Bar{\mathfrak{l}}^{\flat}$}}). 
\end{equation*}
\begin{tikzpicture}[overlay,remember picture,>=stealth,nodes={align=left,inner ysep=1pt},<-]
    \path (ignorance.south) ++ (0,0.1em) node[anchor=north east,color=red!67] (igtt){\textit{ ignorance space }};
    \draw [color=red!87](ignorance.south) |- ([xshift=-0.3ex,color=red] igtt.south west);
    \path (knowledge.north) ++ (0,0.5em) node[anchor=south east,color=blue!67] (kntt){\textit{ extra knowledge }};
    \draw [color=blue!87](knowledge.north) |- ([xshift=-0.3ex,color=blue]kntt.south west);
\end{tikzpicture}

Our Theorem~\ref{th:main} thus meaningfully shows that the linear probing error can be bounded more tightly as $\kappa(\Vec{y})$ increases (i.e., when labeled data provides more useful information for the unlabeled data). 

\textbf{Implication of Theorem~\ref{th:main}.} Our theorem allows us to formalize answers to the ``When and How'' question. Firstly, the Theorem answers ``\textit{how the labeled data helps}''---because the knowledge from the known classes changes the representation of unlabeled data and reduces the ignorance space for novel class discovery. Secondly, the Theorem answers ``\textit{when the labeled data helps}''. Specifically, labeled data helps when the coverage between ignorance space and extra knowledge is nonzero. In the extreme case, if the extra knowledge fully covers the ignorance space, we get the perfect performance (0 linear probing error).

\section{Experiments on Common Benchmarks}
\label{sec:exp}
Beyond theoretical insights, we show empirically that our proposed NCD spectral  loss is effective on common benchmark datasets CIFAR-10 and CIFAR-100~\cite{krizhevsky2009learning}. Following the well-established NCD benchmarks ~\cite{Han2019dtc, han2020automatically, fini2021unified}, each dataset is divided into two subsets, the labeled set that
contains labeled images belonging to a set of known classes, and an unlabeled set with novel classes. Our comparison is on three benchmarks: \texttt{C10-5} means CIFAR-10 datasets split with 5 known classes and 5 novel classes and \texttt{C100-80} means CIFAR-100 datasets split with 80 known classes while \texttt{C100-50} has 50 known classes. The division is consistent with~\citet{fini2021unified}. %
We train the model by the proposed NSCL algorithm with details in Appendix~\ref{sec:sup_exp_cifar} and measure performance on the features in the penultimate layer of ResNet-18. 

 \paragraph{NSCL is competitive in discovering novel classes.} Our proposed loss NSCL is amenable to   the theoretical understanding of NCD, which is our primary goal of this work. Beyond theory, we show that NSCL is equally desirable in empirical performance. In particular, NSCL outperforms its rivals by a significant margin, as evidenced in Table~\ref{tab:main}. Our comparison covers an extensive collection of common NCD algorithms and baselines. In particular, on C100-50, we improve upon the best baseline ComEx by \textbf{10.6}\%. 
 This finding further validates that putting analysis on NSCL is appealing for both theoretical and empirical reasons.%

\begin{table}[t]
\caption{Main Results. Results are reported in clustering accuracy (\%) on the \textit{training} split of the novel set. With the learned feature, we perform a K-Means clustering with the default setting in Python's \texttt{sklearn} package. The accuracy of the novel classes is measured by solving an optimal assignment problem using the Hungarian algorithm~\cite{kuhn1955hungarian}. ``C'' is short for CIFAR. SCL denotes training with Spectral Contrastive Loss purely on $\mathcal{D}_u$ while SCL$^\ddagger$ is trained on $\mathcal{D}_u \cup \mathcal{D}_l$ unsupervisedly.  } 
\centering
\scalebox{0.85}{
\begin{tabular}{llll}
\toprule
\textbf{Method} & \textbf{C10-5} & \textbf{C100-80} & \textbf{C100-50} \\ \midrule
 \textbf{KCL}~\cite{hsu2017kcl} & 72.3 & 42.1 & -\\
 \textbf{MCL}~\cite{hsu2019mcl} & 70.9 & 21.5 & - \\
\textbf{DTC}~\cite{Han2019dtc} & 88.7 & 67.3 & 35.9\\
\textbf{RS+}~\cite{zhao2020rankstat} & 91.7 & 75.2 & 44.1 \\
\textbf{DualRank}~\cite{zhao21novel} & 91.6 & 75.3 & - \\
\textbf{Joint}~\cite{jia2021joint} & 93.4 & 76.4 & - \\
\textbf{UNO}~\cite{fini2021unified} & 92.6 & 85.0 & 52.9  \\
\textbf{ComEx}~\cite{yang2022divide} & 93.6 & 85.7 & 53.4
\\ \midrule
\textbf{SCL}~\cite{haochen2021provable}  & 92.4 & 72.7 & 51.8\\
\textbf{SCL}{$^\ddagger$}~\cite{haochen2021provable}  & 93.7  & 68.9 & 53.3\\
\textbf{NSCL} (Ours)  & \textbf{97.5}  &  \textbf{85.9} & \textbf{64.0}
\\ \bottomrule
\end{tabular}}
\label{tab:main}
\end{table}

\begin{table*}[htb]
\caption{Comparison of results reported in overall/novel/known accuracy (\%) on the \textit{test} split of CIFAR. The three metrics are calculated as follows. (1) \textbf{Known accuracy}: For the features from the labeled data, we train an additional linear head by linear probing and then measure classification accuracy based on the prediction $\Vec{h}_l$;  (2) \textbf{Novel accuracy}: For features from the unlabeled data, we perform a K-Means clustering with the default setting in Python's \texttt{sklearn} package, which produces the clustering prediction $\Vec{h}_u$. The clustering accuracy is further measured by solving an optimal assignment problem using the Hungarian algorithm~\cite{kuhn1955hungarian}; (3) \textbf{Overall accuracy}. The overall accuracy is measured by concatenating the prediction $\Vec{h}_l$ and $\Vec{h}_u$ and then solving the assignment problem.} 
\vspace{0.2cm}
\centering
\resizebox{0.65\linewidth}{!}{
\begin{tabular}{lllllll}
\toprule
\multirow{2}{*}{\textbf{Method}} & \multicolumn{3}{c}{\textbf{C10-5}} & \multicolumn{3}{c}{\textbf{C100-50}} \\
 & \textbf{All} & \textbf{Novel} & \textbf{Known} & \textbf{All} & \textbf{Novel} & \textbf{Known} \\ \midrule
 \textbf{DTC}~\cite{Han2019dtc} & 68.7 & 78.6 & 58.7 & 32.5 & 34.7 & 30.2 \\
\textbf{RankStats}~\cite{zhao2020rankstat} & 89.7 & 88.8 & 90.6 & 55.3 & 40.9 & 69.7 \\
\textbf{UNO}~\cite{fini2021unified} & \textbf{95.8} & 95.1 & 96.6  & 65.4  & 52.0 & 78.8 \\
\textbf{ComEx}~\cite{yang2022divide} & 95.0 & 93.2 & \textbf{96.7} & 67.2 & 54.5 & \textbf{80.1} \\ \midrule
\textbf{NSCL} (Ours) & 95.5 & \textbf{96.7} & 94.2  &  \textbf{67.4} & \textbf{57.1} &	77.4
\\ \bottomrule
\end{tabular}}
\label{tab:sup_main}
\end{table*}

\paragraph{Ablation study on the unsupervised counterpart.} To verify whether the known classes indeed help discover new classes, we compare NSCL with the unsupervised counterpart (dubbed SCL) that is purely trained on the unlabeled data $\mathcal{D}_u$. Results show that the labeled data offers tremendous help and improves \textbf{13.2}\% in novel class accuracy. 

\paragraph{Supervision signals are important in the labeled data.} We also analyze how much the supervision signals in labeled data help. To investigate it, we compare our method NSCL with SCL trained on $\mathcal{D}_u \cup \mathcal{D}_l$ in a purely unsupervised manner. The difference is that SCL does not utilize the label information in $\mathcal{D}_l$. We denote this setting as SCL$^\ddagger$ in Table~\ref{tab:main}. Results show that NSCL provides stronger performance than SCL$^\ddagger$. The ablation suggests that relevant knowledge of known classes indeed provides meaningful help in novel class discovery. 

\paragraph{NSCL is competitive in the inductive setting.} 
We report performance comparison in Table~\ref{tab:sup_main}, comprehensively measuring three accuracy metrics---for all/novel/known classes respectively. Different from Table~\ref{tab:main} which reports clustering results in a transductive manner, the performance in Table~\ref{tab:sup_main} is reported on the test split.   
For evaluation, we first collect the feature representations and then report overall/novel/known accuracy with inference details provided in the caption of Table~\ref{tab:sup_main}. 
We see that NSCL establishes comparable performance with baselines on the labeled data from known classes  and superior performance on novel class discovery. Notably, NSCL outperforms UNO~\cite{fini2021unified} on \texttt{C10-5} by 1.6\% and outperforms ComEx~\cite{yang2022divide} by 2.6\% on \texttt{C100-50} in terms of novel accuracy.

\section{Conclusion}
In this paper, we present a theoretical framework of novel class discovery and provide new insight on the research question: ``\textit{when and how does the known class help discover novel classes?}''. Specifically, we propose a graph-theoretic representation that can be learned through a new NCD Spectral Contrastive Loss (NSCL). Minimizing this objective is equivalent to factoring the graph's adjacency matrix, which allows us to analyze the NCD quality by measuring the linear probing error on novel samples' features. Our main result (Theorem~\ref{th:no_approx}) suggests such error can be significantly reduced (even to 0) when the linear span of known samples' feature covers the ``ignorance space'' of unlabeled data in discovering novel classes. Our framework is also empirically appealing to use since it can achieve similar or better performance than existing methods on benchmark datasets.

\textbf{Broader impacts.} Our new framework opens a new door to the NCD community in the following way:

\begin{itemize}[noitemsep,topsep=0pt,parsep=0pt,partopsep=0pt]
    \item NSCL provides a framework to answer the fundamental question that is shared across all NCD methods. At a high level, NSCL analyzes how the new knowledge changes the representation space that leads to different discovery outcomes. This finding can be generalizable to other NCD methods which may differ in the way of incorporating new knowledge.
    \item NSCL can be compatible with prior NCD methods. Note that NSCL is a representation learning method. With that being said, one can possibly ``plug'' NSCL into existing learning objectives for NCD. Take the two most popular prior works in NCD as an example. For example, we can use the encoder learned by NSCL in RS+~\cite{zhao2020rankstat} and UNO~\cite{fini2021unified}.
\end{itemize}

To summarize, NSCL is an important building block  in the NCD research area and have broader impacts both theoretically and empirically.

\section*{Acknowledgement}
\label{sec:ack}

Li is supported in part by the AFOSR Young Investigator Award under No. FA9550-23-1-0184; Philanthropic Fund from SFF; and faculty research awards/gifts from Google, Meta, and Amazon. Liang is partially supported by Air Force Grant FA9550-18-1-0166, the National Science Foundation (NSF) Grants 2008559-IIS and CCF-2046710. Any opinions, findings, conclusions, or recommendations
expressed in this material are those of the authors and do not necessarily reflect the views, policies, or endorsements either expressed or implied, of the sponsors. The authors would also like to thank ICML
reviewers for their helpful suggestions and feedback.

\newpage
\clearpage
\bibliography{main}
\bibliographystyle{icml2023}


\newpage
\clearpage
\onecolumn
\appendix
\begin{center}
	\textbf{\LARGE Appendix }
\end{center}

\section{Proof Details for Section~\ref{sec:method}}
\label{sec:sup_method_proof}

\subsection{Bound Linear Probing Error by Regression Residual}
\label{sec:sup_cls_bound}

\begin{lemma} (Recap of Lemma~\ref{lemma:cls_bound})
Denote by $\mathbf{y}(x) \in \mathbb{R}^{|\mathcal{Y}_u|}$  a one-hot vector, whose $y(x)$-th position is 1 and 0 elsewhere. Let
$\mathbf{Y} \in \mathbb{R}^{N_u \times |\mathcal{Y}_u|}$ be a matrix whose rows are stacked by $\mathbf{y}(x)$. We have: 
    $$\mathcal{R}(U^*) \triangleq \underset{{\*M}\in \mathbb{R}^{k \times |\mathcal{Y}_u|}}{\operatorname{min}} \|\mathbf{Y} - U^* \*M \|^2_F \geq \frac{1}{2}\mathcal{E}(f) $$
\label{lemma_sup:cls_bound}
\end{lemma}
\begin{proof}
    Suppose $\Tilde{f}(x) = \sqrt{w_x}f(x)$, we first show that 
\begin{align*}
\| \mathbf{y}(x) - \Tilde{f}(x)^{\top} \*M \|^2  &\geq \frac{1}{2} \mathbbm{1}\left[y(x) \neq h(x;\Tilde{f}, M)\right] 
\end{align*}
If $y(x) = h(x;\Tilde{f}, M)$, it is clear that $\| \mathbf{y}(x) - \Tilde{f}(x)^{\top} \*M \|^2 \geq 0$. If $y(x) \neq h(x;\Tilde{f}, M)$, then there exists another index $y' \neq y(x)$ so that $\Tilde{f}(x)^{\top} \Vec{\mu}_{y'} \geq \Tilde{f}(x)^{\top} \Vec{\mu}_{y(x)}$. Then, 
\begin{align*}
\| \mathbf{y}(x) - \Tilde{f}(x)^{\top} \*M \|_2^2 &\geq (1 - \Tilde{f}(x)^{\top} \Vec{\mu}_{y(x)})^2 + (\Tilde{f}(x)^{\top} \Vec{\mu}_{y'})^2 
\\ &\geq \frac{1}{2} (1 - \Tilde{f}(x)^{\top} \Vec{\mu}_{y(x)} + \Tilde{f}(x)^{\top} \Vec{\mu}_{y'})^2
\\ &\geq \frac{1}{2},
\end{align*}
where the first inequality is by only keeping $y'$-th and $y(x)$-th terms in the $l_2$ norm. We can then prove the lemma by: 
\begin{align*}
    \mathcal{R}(U^*) &= \underset{{\*M}\in \mathbb{R}^{k \times |\mathcal{Y}_u|}}{\operatorname{min}} \|\mathbf{Y} - U^* \*M \|^2_F    \\
    &= \underset{{\*M}\in \mathbb{R}^{k \times |\mathcal{Y}_u|}}{\operatorname{min}} 
    \underset{{x \in \mathcal{X}_u}}{\sum}
    \| \mathbf{y}(x) - \sqrt{w_x}f(x)^{\top} \Sigma_k^{-\frac{1}{2}}\*M \|^2     \\ 
    &= \underset{{\*M}\in \mathbb{R}^{k \times |\mathcal{Y}_u|}}{\operatorname{min}} 
    \underset{{x \in \mathcal{X}_u}}{\sum}
    \| \mathbf{y}(x) - \sqrt{w_x}f(x)^{\top} \*M \|^2     \\ 
    &\geq \frac{1}{2} \underset{{\*M}\in \mathbb{R}^{k \times |\mathcal{Y}_u|}}{\operatorname{min}} 
    \underset{{x \in \mathcal{X}_u}}{\sum} \mathbbm{1}\left[y(x) \neq h(x;\Tilde{f}, M)\right] \\
    &= \frac{1}{2}\mathcal{E}(f),
\end{align*}
 where the second equation is given by $F^* \Sigma_k^{-\frac{1}{2}} = V_k$, and $U^*$ is the last $N_u$ rows of $V_k$, and the last equation is based on the fact that multiplying a scalar value on the output does not change the prediction result ($h(x;f, \*M) = h(x;\Tilde{f}, \*M)$). 
\end{proof}

\subsection{Spectral Contrastive Loss}
\label{sec:proof-nscl}
\begin{theorem}
\label{th:sup-ncd-scl} (Recap of Theorem~\ref{th:ncd-scl}) 
We define $\*f_x = \sqrt{w_x}f(x)$ for some function $f$. Recall $\alpha,\beta$ is a hyper-parameter defined in Eq.~\eqref{eq:def_wxx}. Then minimizing the loss function $\mathcal{L}_{\mathrm{mf}}(F, A)$ is equivalent to minimizing the following loss function for $f$, which we term \textbf{NCD Spectral Contrastive Loss (NSCL)}:
\begin{align}
\begin{split}
    \mathcal{L}_{nscl}(f) &\triangleq - 2\alpha \mathcal{L}_1(f) 
- 2\beta  \mathcal{L}_2(f) \\ & + \alpha^2 \mathcal{L}_3(f) + 2\alpha \beta \mathcal{L}_4(f) +  
\beta^2 \mathcal{L}_5(f),
\label{eq:sup_def_nscl}
\end{split}
\end{align}
where 
\begin{align*}
    \mathcal{L}_1(f) &= \sum_{i \in \mathcal{Y}_l}\underset{\substack{\bar{x}_{l} \sim \mathcal{P}_{{l_i}}, \bar{x}'_{l} \sim \mathcal{P}_{{l_i}},\\x \sim \mathcal{T}(\cdot|\bar{x}_{l}), x^{+} \sim \mathcal{T}(\cdot|\bar{x}'_l)}}{\mathbb{E}}\left[f(x)^{\top} {f}\left(x^{+}\right)\right] , 
    \mathcal{L}_2(f) = \underset{\substack{\bar{x}_{u} \sim \mathcal{P}_{u},\\x \sim \mathcal{T}(\cdot|\bar{x}_{u}), x^{+} \sim \mathcal{T}(\cdot|\bar{x}_u)}}{\mathbb{E}}
\left[f(x)^{\top} {f}\left(x^{+}\right)\right], \\
    \mathcal{L}_3(f) &= \sum_{i \in \mathcal{Y}_l}\sum_{j \in \mathcal{Y}_l}\underset{\substack{\bar{x}_l \sim \mathcal{P}_{{l_i}}, \bar{x}'_l \sim \mathcal{P}_{{l_j}},\\x \sim \mathcal{T}(\cdot|\bar{x}_l), x^{-} \sim \mathcal{T}(\cdot|\bar{x}'_l)}}{\mathbb{E}}
\left[\left(f(x)^{\top} {f}\left(x^{-}\right)\right)^2\right], 
    \mathcal{L}_4(f) = \sum_{i \in \mathcal{Y}_l}\underset{\substack{\bar{x}_l \sim \mathcal{P}_{{l_i}}, \bar{x}_u \sim \mathcal{P}_{u},\\x \sim \mathcal{T}(\cdot|\bar{x}_l), x^{-} \sim \mathcal{T}(\cdot|\bar{x}_u)}}{\mathbb{E}}
\left[\left(f(x)^{\top} {f}\left(x^{-}\right)\right)^2\right], \\
    \mathcal{L}_5(f) &= \underset{\substack{\bar{x}_u \sim \mathcal{P}_{u}, \bar{x}'_u \sim \mathcal{P}_{u},\\x \sim \mathcal{T}(\cdot|\bar{x}_u), x^{-} \sim \mathcal{T}(\cdot|\bar{x}'_u)}}{\mathbb{E}}
\left[\left(f(x)^{\top} {f}\left(x^{-}\right)\right)^2\right].
\end{align*}
\end{theorem}
\begin{proof} We can expand $\mathcal{L}_{\mathrm{mf}}(F, A)$ and obtain
\begin{align*}
&\mathcal{L}_{\mathrm{mf}}(F, A) =\sum_{x, x^{\prime} \in \mathcal{X}}\left(\frac{w_{x x^{\prime}}}{\sqrt{w_x w_{x^{\prime}}}}-\*f_x^{\top} \*f_{x^{\prime}}\right)^2 = \text{const} + \sum_{x, x^{\prime} \in \mathcal{X}}\left(-2 w_{x x^{\prime}} f(x)^{\top} {f}\left(x^{\prime}\right)+w_x w_{x^{\prime}}\left(f(x)^{\top}{f}\left(x^{\prime}\right)\right)^2\right),
\end{align*} 
where $\*f_x = \sqrt{w_x}f(x)$ is a re-scaled version of $f(x)$.
At a high level we follow the proof in ~\cite{haochen2021provable}, while the specific form of loss varies with the different definitions of positive/negative pairs. The form of $\mathcal{L}_{nscl}(f)$ is derived from plugging $w_{xx'}$  and $w_x$. 

Recall that $w_{xx'}$ is defined by
\begin{align*}
w_{x x^{\prime}} &= \alpha \sum_{i \in \mathcal{Y}_l}\mathbb{E}_{\bar{x}_{l} \sim {\mathcal{P}_{l_i}}} \mathbb{E}_{\bar{x}'_{l} \sim {\mathcal{P}_{l_i}}} \mathcal{T}(x | \bar{x}_{l}) \mathcal{T}\left(x' | \bar{x}'_{l}\right)+ \beta \mathbb{E}_{\bar{x}_{u} \sim {\mathcal{P}_u}} \mathcal{T}(x| \bar{x}_{u}) \mathcal{T}\left(x'| \bar{x}_{u}\right) ,
\end{align*}
and $w_{x}$ is given by 
\begin{align*}
w_{x } &= \sum_{x^{\prime}}w_{xx'} \\ &=\alpha \sum_{i \in \mathcal{Y}_l}\mathbb{E}_{\bar{x}_{l} \sim {\mathcal{P}_{l_i}}} \mathbb{E}_{\bar{x}'_{l} \sim {\mathcal{P}_{l_i}}} \mathcal{T}(x | \bar{x}_{l}) \sum_{x^{\prime}} \mathcal{T}\left(x' | \bar{x}'_{l}\right)+ \beta \mathbb{E}_{\bar{x}_{u} \sim {\mathcal{P}_u}} \mathcal{T}(x| \bar{x}_{u}) \sum_{x^{\prime}} \mathcal{T}\left(x' | \bar{x}_{u}\right) \\
&= \alpha \sum_{i \in \mathcal{Y}_l}\mathbb{E}_{\bar{x}_{l} \sim {\mathcal{P}_{l_i}}} \mathcal{T}(x | \bar{x}_{l}) + \beta \mathbb{E}_{\bar{x}_{u} \sim {\mathcal{P}_u}} \mathcal{T}(x| \bar{x}_{u}). 
\end{align*}

Plugging $w_{x x^{\prime}}$ we have, 

\begin{align*}
    &-2 \sum_{x, x^{\prime} \in \mathcal{X}} w_{x x^{\prime}} f(x)^{\top} {f}\left(x^{\prime}\right) = -2 \sum_{x, x^{+} \in \mathcal{X}} w_{x x^{+}} f(x)^{\top} {f}\left(x^{+}\right) 
    \\ &= -2\alpha  \sum_{i \in \mathcal{Y}_l}\mathbb{E}_{\bar{x}_{l} \sim {\mathcal{P}_{l_i}}} \mathbb{E}_{\bar{x}'_{l} \sim {\mathcal{P}_{l_i}}} \sum_{x, x^{\prime} \in \mathcal{X}} \mathcal{T}(x | \bar{x}_{l}) \mathcal{T}\left(x' | \bar{x}'_{l}\right) f(x)^{\top} {f}\left(x^{\prime}\right) -2 \beta \mathbb{E}_{\bar{x}_{u} \sim {\mathcal{P}_u}} \sum_{x, x^{\prime}} \mathcal{T}(x| \bar{x}_{u}) \mathcal{T}\left(x'| \bar{x}_{u}\right) f(x)^{\top} {f}\left(x^{\prime}\right) 
    \\ &= -2\alpha  \sum_{i \in \mathcal{Y}_l}\underset{\substack{\bar{x}_{l} \sim \mathcal{P}_{{l_i}}, \bar{x}'_{l} \sim \mathcal{P}_{{l_i}},\\x \sim \mathcal{T}(\cdot|\bar{x}_{l}), x^{+} \sim \mathcal{T}(\cdot|\bar{x}'_l)}}{\mathbb{E}}  \left[f(x)^{\top} {f}\left(x^{+}\right)\right] - 2\beta 
    \underset{\substack{\bar{x}_{u} \sim \mathcal{P}_{u},\\x \sim \mathcal{T}(\cdot|\bar{x}_{u}), x^{+} \sim \mathcal{T}(\cdot|\bar{x}_u)}}{\mathbb{E}}
\left[f(x)^{\top} {f}\left(x^{+}\right)\right] = - 2\alpha  \mathcal{L}_1(f) 
- 2\beta  \mathcal{L}_2(f)
\end{align*}

Plugging $w_{x}$ and $w_{x'}$ we have, 

\begin{align*}
    &\sum_{x, x^{\prime} \in \mathcal{X}}w_x w_{x^{\prime}}\left(f(x)^{\top}{f}\left(x^{\prime}\right)\right)^2 = \sum_{x, x^{-} \in \mathcal{X}}w_x w_{x^{-}}\left(f(x)^{\top}{f}\left(x^{-}\right)\right)^2 \\
    &=\sum_{x, x^{\prime} \in \mathcal{X}} \left( \alpha \sum_{i \in \mathcal{Y}_l}\mathbb{E}_{\bar{x}_{l} \sim {\mathcal{P}_{l_i}}} \mathcal{T}(x | \bar{x}_{l}) + \beta \mathbb{E}_{\bar{x}_{u} \sim {\mathcal{P}_u}} \mathcal{T}(x| \bar{x}_{u}) \right) \cdot \\
    &~~~~~~~~~~~~~~~~~~~~~~~~~~~~\left(\alpha  \sum_{j \in \mathcal{Y}_l}\mathbb{E}_{\bar{x}'_{l} \sim {\mathcal{P}_{l_j}}} \mathcal{T}(x^{-} | \bar{x}'_{l}) + \beta \mathbb{E}_{\bar{x}'_{u} \sim {\mathcal{P}_u}} \mathcal{T}(x^{-}| \bar{x}'_{u}) \right) \left(f(x)^{\top}{f}\left(x^{-}\right)\right)^2 \\
    &= \alpha ^ 2 \sum_{x, x^{-} \in \mathcal{X}}  \sum_{i \in \mathcal{Y}_l}\mathbb{E}_{\bar{x}_{l} \sim {\mathcal{P}_{l_i}}} \mathcal{T}(x | \bar{x}_{l}) \sum_{j \in \mathcal{Y}_l}\mathbb{E}_{\bar{x}'_{l} \sim {\mathcal{P}_{l_j}}} \mathcal{T}(x^{-} | \bar{x}'_{l})\left(f(x)^{\top}{f}\left(x^{-}\right)\right)^2 \\
    & + 2\alpha \beta \sum_{x, x^{-} \in \mathcal{X}} \sum_{i \in \mathcal{Y}_l}\mathbb{E}_{\bar{x}_{l} \sim {\mathcal{P}_{l_i}}} \mathcal{T}(x | \bar{x}_{l})  \mathbb{E}_{\bar{x}_{u} \sim {\mathcal{P}_u}} \mathcal{T}(x^{-}| \bar{x}_{u}) \left(f(x)^{\top}{f}\left(x^{-}\right)\right)^2  \\
    &+ \beta^2 \sum_{x, x^{-} \in \mathcal{X}}  \mathbb{E}_{\bar{x}_{u} \sim {\mathcal{P}_u}} \mathcal{T}(x| \bar{x}_{u}) \mathbb{E}_{\bar{x}'_{u} \sim {\mathcal{P}_u}} \mathcal{T}(x^{-}| \bar{x}'_{u}) \left(f(x)^{\top}{f}\left(x^{-}\right)\right)^2 \\
    &= \alpha^2 \sum_{i \in \mathcal{Y}_l}\sum_{j \in \mathcal{Y}_l}\underset{\substack{\bar{x}_l \sim \mathcal{P}_{{l_i}}, \bar{x}'_l \sim \mathcal{P}_{{l_j}},\\x \sim \mathcal{T}(\cdot|\bar{x}_l), x^{-} \sim \mathcal{T}(\cdot|\bar{x}'_l)}}{\mathbb{E}}
\left[\left(f(x)^{\top} {f}\left(x^{-}\right)\right)^2\right] + 2\alpha\beta
    \sum_{i \in \mathcal{Y}_l}\underset{\substack{\bar{x}_l \sim \mathcal{P}_{{l_i}}, \bar{x}_u \sim \mathcal{P}_{u},\\x \sim \mathcal{T}(\cdot|\bar{x}_l), x^{-} \sim \mathcal{T}(\cdot|\bar{x}_u)}}{\mathbb{E}}
\left[\left(f(x)^{\top} {f}\left(x^{-}\right)\right)^2\right] \\ &+ \beta^2
     \underset{\substack{\bar{x}_u \sim \mathcal{P}_{u}, \bar{x}'_u \sim \mathcal{P}_{u},\\x \sim \mathcal{T}(\cdot|\bar{x}_u), x^{-} \sim \mathcal{T}(\cdot|\bar{x}'_u)}}{\mathbb{E}}
\left[\left(f(x)^{\top} {f}\left(x^{-}\right)\right)^2\right] 
\\& = \alpha^2 \mathcal{L}_3(f) + 2\alpha\beta \mathcal{L}_4(f) + \beta^2\mathcal{L}_5(f).
\end{align*}
\end{proof}

\newpage
\section{Proof for Eigenvalue in Toy Example }
\label{sec:proof-eigen}
Before we present the proof of Theorem~\ref{th:toy_extreme}, Theorem\ref{th:toy_general} and Lemma~\ref{th:toy_harmful}, we first present the following lemma~\ref{lemma:sup_eigprob} which extensively explore the order and the form of eigenvectors of the general form $T(t)$. 
Note that $T_1$ and $T_2$ are special cases of the following $T(t)$ with $t \in [\tau_0, \tau_c]$:

\begin{equation*}
T(t)=\left[\begin{array}{ccccc}
\tau_1 & t & t & \tau_0 & \tau_0 \\
t & \tau_1 & \tau_{c} & \tau_{s} & \tau_0  \\
t & \tau_{c} & \tau_1 & \tau_0 & \tau_{s}  \\
\tau_0 & \tau_{s} & \tau_0 & \tau_1 & \tau_{c}  \\
\tau_0 & \tau_0 & \tau_{s} & \tau_{c} & \tau_1  \\
\end{array}\right],
\end{equation*}
where $t$ indicates the strength of the connection between labeled data and a novel class in unlabeled data.

\begin{lemma}
    Assume $\tau_1 = 1$, $\tau_0 = 0$, $\tau_c < \tau_s < 1.5\tau_c$, $\Bar{t} = \sqrt{\frac{2(\tau_s-\tau_c)^2\tau_c}{2\tau_c - \tau_s}}$, let $a(\lambda)  = {\lambda-1 \over 2t}$ and $b(\lambda)  = {\tau_{s}(\lambda-1) \over 2(\lambda -1 -\tau_{c})t}$ are real value functions, the matrix $T(t)$'s eigenvectors (not necessarily $l_2$-normalized) and its eigenvalues are the following:
    
    (Case 1): If  $t \in (\Bar{t}, \tau_c]$,  
    \begin{align*}
    \begin{array}{ll}
        v_1 = [1, a(\lambda_1), a(\lambda_1), b(\lambda_1), b(\lambda_1)]^{\top}, &\lambda_1 > 1 + \tau_{s} + \tau_{c}, \\
        v_2 = [1, a(\lambda_2), a(\lambda_2), b(\lambda_2), b(\lambda_2)]^{\top}, &\lambda_2 \in [1 + \tau_{s} - \tau_{c}, 1 + \tau_{c}) \\
        v_3 = [0, -1, 1, -1, 1]^{\top}, &\lambda_3 = 1 + \tau_{s} - \tau_{c}, \\
        v_4 = [1, a(\lambda_4), a(\lambda_4), b(\lambda_4), b(\lambda_4)]^{\top}, &\lambda_4 \in (1 - \tau_{s} - \tau_{c},  1) \\
        v_5 = [0, 1, -1, -1, 1]^{\top}, &\lambda_5 = 1 - \tau_{s} - \tau_{c}, \\
    \end{array}
    \end{align*}
    
    (Case 2): If  $t \in (0, \Bar{t})$, 
        \begin{align*}
    \begin{array}{ll}
        v_1 = [1, a(\lambda_1), a(\lambda_1), b(\lambda_1), b(\lambda_1)]^{\top}, &\lambda_1 > 1 + \tau_{s} + \tau_{c}, \\
        v_2 = [0, -1, 1, -1, 1]^{\top}, &\lambda_2 = 1 + \tau_{s} - \tau_{c}, \\
        v_3 = [1, a(\lambda_3), a(\lambda_3), b(\lambda_3), b(\lambda_3)]^{\top}, &\lambda_3 \in [1, 1 + \tau_{s} - \tau_{c}) \\
        v_4 = [1, a(\lambda_4), a(\lambda_4), b(\lambda_4), b(\lambda_4)]^{\top}, &\lambda_4 \in (1 - \tau_{s} - \tau_{c},1) \\
        v_5 = [0, 1, -1, -1, 1]^{\top}, &\lambda_5 = 1 - \tau_{s} - \tau_{c}, \\
    \end{array}
    \end{align*}

        (Case 3): If  $t = 0$, 
        \begin{align*}
    \begin{array}{ll}
        v_1 = [0, 1, 1, 1, 1]^{\top}, &\lambda_1 = 1 + \tau_{s} + \tau_{c}, \\
        v_2 = [0, -1, 1, -1, 1]^{\top}, &\lambda_2 = 1 + \tau_{s} - \tau_{c}, \\
        v_3 = [1, 0, 0, 0, 0]^{\top}, &\lambda_3 = 1 \\
        v_4 = [0, 1, 1, -1, -1]^{\top}, &\lambda_4 = 1 - \tau_{s} + \tau_{c} \\
        v_5 = [0, 1, -1, -1, 1]^{\top}, &\lambda_5 = 1 - \tau_{s} - \tau_{c}, \\
    \end{array}
    \end{align*}
    \label{lemma:sup_eigprob}
\end{lemma}

\begin{proof}
For $t=0$, Case 3, we can verify by direct calculation. 

Now for Case 1 and Case 2, we consider $t\in (0, \tau_{c})$. For  any $i \in [5]$, denote $\hat{\lambda}_i$ as unordered eigenvalue and $\hat{v_i}$ is its corresponding eigenvector.
We can direct verify that
\begin{align}
    \hat{\lambda}_1 =& 1 + \tau_{s} - \tau_{c}\\
    \hat{\lambda}_2 = & 1 - \tau_{s} - \tau_{c},
\end{align}
are two eigenvalues of $\Tilde{A}_t$ and 
\begin{align}
    \hat{v}_1 =& [0, -1, 1, -1, 1]^{\top}\\
    \hat{v}_2 = &[0, 1, -1, -1, 1]^{\top}, 
\end{align}
are two corresponding eigenvectors. Now, we prove for $i \in \{3,4,5\}$,  $\hat{v}_i = [1, a(\hat{\lambda}_i), a(\hat{\lambda}_i), b(\hat{\lambda}_i), b(\hat{\lambda}_i)]^{\top}$ are eigenvector for $\hat{\lambda}_i$.
For $i \in \{3,4,5\}$ we only need to show
\begin{align}
    \begin{cases}
    1+2t a(\hat{\lambda}_i) & = \hat{\lambda}_i \\
    t + (1+\tau_{c})a(\hat{\lambda}_i) + \tau_{s}b(\hat{\lambda}_i) & = \hat{\lambda}_i a(\hat{\lambda}_i) \\
    \tau_{s}a(\hat{\lambda}_i) + (1+\tau_{c})b(\hat{\lambda}_i) & = \hat{\lambda}_i b(\hat{\lambda}_i). 
    \end{cases}
\end{align}
Equivalently to
\begin{align}
    \begin{cases}
    1+2t a(\hat{\lambda}_i) -\hat{\lambda}_i & = 0 \\
    t + (1+\tau_{c}+\tau_{s}-\hat{\lambda}_i)(a(\hat{\lambda}_i)d + b(\hat{\lambda}_i)) & = 0 \\
    t + (1+\tau_{c}-\tau_{s}-\hat{\lambda}_i)(a(\hat{\lambda}_i) - b(\hat{\lambda}_i)) & = 0 .
    \end{cases}
\end{align}
Let $z_i = \hat{\lambda}_i - 1$. Equivalently to
\begin{align}
    \begin{cases}
    1+2t a(\hat{\lambda}_i) -\hat{\lambda}_i & = 0 \\
    (\hat{\lambda}_i -1 -\tau_{c})b(\hat{\lambda}_i) -\tau_{s}a(\hat{\lambda}_i)  & = 0\\
    z_i^3-2\tau_{c}z_i^2+(\tau_{c}^2-\tau_{s}^2-2t^2)z_i+2\tau_{c}t^2 &=0 .
    \end{cases}
\end{align}
Let $g(z) = z^3-2\tau_{c}z^2+(\tau_{c}^2-\tau_{s}^2-2t^2)z+2\tau_{c}t^2$, we can verify that $g(-\infty)<0, ~~ g(-\tau_{c}-\tau_{s}) = -4\tau_{c}(\tau_{c}+\tau_{s})^2+4t^2\tau_{c}+2t^2\tau_{s} < 0, ~~ g(0) = 2\tau_{c}t^2 > 0, ~~ g(\tau_{c}) = -\tau_{s}^2\tau_{c}<0 , ~~g(\tau_{c}+\tau_{s}) = -2\tau_{s}t^2 < 0 , ~~g(+\infty)>0$. Thus, we have three solutions and satisfying $1-\tau_{c}-\tau_{s} < \hat{\lambda}_5 < 1 < \hat{\lambda}_4 < 1+\tau_{c}< 1+\tau_{c}+\tau_{s}<\hat{\lambda}_3$. As $\hat{\lambda}_i \neq 1+\tau_{c}$ for $i \in \{3,4,5\}$, thus,  equivalently to
\begin{align}
    \begin{cases}
    a(\hat{\lambda}_i)  & = {\hat{\lambda}_i-1 \over 2t} \\
    b(\hat{\lambda}_i)  & = {\tau_{s}(\hat{\lambda}_i-1) \over 2(\hat{\lambda}_i -1 -\tau_{c})t}\\
    (\hat{\lambda}_i-1)^3-2\tau_{c}(\hat{\lambda}_i-1)^2+(\tau_{c}^2-\tau_{s}^2-2t^2)(\hat{\lambda}_i-1)+2\tau_{c}t^2 &=0 .
    \end{cases}
\end{align}

When $t > \Bar{t}$, we have $g(\tau_{s} - \tau_{c})>0$. Thus, we have $1-\tau_{c}-\tau_{s} < \hat{\lambda}_5 < 1+\tau_{s} - \tau_{c} < \hat{\lambda}_4 < 1+\tau_{c}+\tau_{s}<\hat{\lambda}_3$. By reorder, we finish Case 1.

When $t < \Bar{t}$, we have $g(\tau_{s} - \tau_{c})<0$. Thus, we have $1-\tau_{c}-\tau_{s} < \hat{\lambda}_5 < 1 < \hat{\lambda}_4 < 1+\tau_{s} - \tau_{c}< 1+\tau_{c}+\tau_{s}<\hat{\lambda}_3$. By reorder the eigenvectors w.r.t the size of eigenvalues, we finish Case 2.
\end{proof}

\begin{theorem}
   (Recap of Theorem~\ref{th:toy_extreme}) Assume $\tau_1 = 1$, $\tau_0 = 0$, $\tau_s < 1.5\tau_c$. We have
$$
U^*_1=\left[\begin{array}{ccccc}
 a_1 & a_1 & b_1 & b_1 \\
 a_2 & a_2 & b_2 & b_2 \\
\end{array}\right]^{\top},
$$
where $a_1,b_1$ are some positive real numbers, and $a_2,b_2$ has different signs. 
$$U^*_2=\left\{\begin{array}{ll}   
\frac{1}{2}\left[\begin{array}{cccc}
1 & 1 & 1 & 1\\ 1 & 1 & -1 & -1\\
\end{array}\right]^{\top}, & \text{if } \tau_{s} < \tau_{c}, \\
\frac{1}{2}\left[\begin{array}{cccc}
1 & 1 & 1 & 1\\ -1 & 1 & -1 & 1\\
\end{array}\right]^{\top}, & \text{if } \tau_{s} > \tau_{c}, 
\end{array}\right. $$
With label vector $\Vec{y} = \{1,1,0,0\}$, we have 
\begin{equation}
    \mathcal{R}(U^*_1, \Vec{y}) = 0, \mathcal{R}(U^*_2, \Vec{y}) = \left\{\begin{array}{ll}    
    0, & \text{if } \tau_{s} < \tau_{c}\\
     1, &  \text{if } \tau_{s} > \tau_{c}.
    \end{array}\right. 
\end{equation}
    \label{th:sup_toy_extreme}
\end{theorem}

\begin{proof}
    In the Case 1 and Case 3 of Lemma~\ref{lemma:sup_eigprob}, we have shown the $U^*_1$ and $U^*_2$ case when $\tau_s > \tau_c$ respectively. In this proof, we just need to show the case when $\tau_s < \tau_c$. For $U^*_2$ and $\tau_s < \tau_c$,  since  $t = 0$, we can directly prove by giving the eigenvectors with order: 
        \begin{align*}
    \begin{array}{ll}
        v_1 = [0, 1, 1, 1, 1]^{\top}, &\lambda_1 = 1 + \tau_{s} + \tau_{c}, \\
        v_2 = [0, 1, 1, -1, -1]^{\top}, &\lambda_2 = 1 - \tau_{s} + \tau_{c} \\
        v_3 = [1, 0, 0, 0, 0]^{\top}, &\lambda_3 = 1 \\
        v_4 = [0, -1, 1, -1, 1]^{\top}, &\lambda_4 = 1 + \tau_{s} - \tau_{c}, \\
        v_5 = [0, 1, -1, -1, 1]^{\top}, &\lambda_5 = 1 - \tau_{s} - \tau_{c}, \\
    \end{array}
    \end{align*}
    
    For $U_1^*$, one can see that in the Case 1 of Lemma~\ref{lemma:sup_eigprob}, we still have $\lambda_2 > \lambda_3$ since $\tau_s<1.5\tau_c<2\tau_c$ holds. Therefore the order of $v_2$ and $v_3$ does not change. Then $U_1^*$ is the concatenation of the last four dimensions of $v_2$ and $v_1$.
    
    Now we would like to show that $a_1, b_1$ are positive and $a_2, b_2$ have different signs. We have shown in Lemma~\ref{lemma:sup_eigprob} that $a(\lambda)  = {\lambda-1 \over 2t}$ and $b(\lambda)  = {\tau_{s}(\lambda-1) \over 2(\lambda -1 -\tau_{c})t}$. Since $a_1 = a(\lambda_1)$ and $b_1 = b(\lambda_1)$, one can show that $a_1 > 0,b_1>0$ since $\lambda_1 > 1 + \tau_s + \tau_c$. For $\lambda_2 \in [1 + \tau_{s} - \tau_{c}, 1 + \tau_{c})$, it is clear that $a_2 = a(\lambda_2) > 0 > b(\lambda_2) = b_2$ when $\tau_{s} > \tau_{c}$, and conversely we have $a_2 = a(\lambda_2) < 0 < b(\lambda_2) = b_2$ when $\tau_{s} < \tau_{c}$. So $a_2$ and $b_2$ have different signs in both cases. 

    Recall $\mathcal{R}(U^*, \Vec{y})$ is defined as:
    $$ \mathcal{R}(U^*, \Vec{y}) =  \underset{{\Vec{\mu}}\in \mathbb{R}^{k}}{\operatorname{min}} \|\Vec{y} - U^* \Vec{\mu} \|^2_2, $$
    Let $\Vec{\mu} = [\frac{b_2}{a_1b_2 - a_2b_1}, \frac{-b_1}{a_1b_2 - a_2b_1}]^{\top}$, $\mathcal{R}(U_1^*, \Vec{y}) = 0$. If $\tau_s < \tau_c$, let $\Vec{\mu} = [1, 1]^{\top}$, then $\mathcal{R}(U_2^*, \Vec{y}) = 0$. If $\tau_s > \tau_c$, $\Vec{\mu}^* = U_2^{*\top} \Vec{y} = [1, 0]^{\top}$ is the minimizer and we have $\mathcal{R}(U_2^*, \Vec{y}) = 1$. 
    
\end{proof}

\begin{theorem}
     (Recap of Theorem~\ref{th:toy_general}) Assume  $\tau_1 = 1$, $\tau_0 = 0$, $1.5\tau_c > \tau_s > \tau_c$. Let $\Bar{t} = \sqrt{\frac{2(\tau_s-\tau_c)^2\tau_c}{2\tau_c - \tau_s}}$, $r: \mathbb{R} \mapsto (0,1) $ as a real value function, we have 
     \begin{equation}
         \mathcal{R}(U^*_t, \Vec{y}) = \left\{\begin{array}{ll}    
     0, &  \text{if } t \in (\Bar{t}, \tau_s), \\
    r(t), & \text{if } t \in (0, \Bar{t}) \\ 
     1, &  \text{if } t = 0. \end{array}\right. 
     \end{equation}
    \label{th:sup_toy_general}
\end{theorem}
\begin{proof}
    According to Lemma~\ref{lemma:sup_eigprob}, if $t \in (\Bar{t}, \tau_s)$, $$
U^*_t=\left[\begin{array}{ccccc}
 a_1 & a_1 & b_1 & b_1 \\
 a_2 & a_2 & b_2 & b_2 \\
\end{array}\right]^{\top},
$$
where $a_1,b_1$ are some positive real numbers, and $a_2,b_2$ has different signs. Let $\Vec{\mu} = [\frac{b_2}{a_1b_2 - a_2b_1}, \frac{-b_1}{a_1b_2 - a_2b_1}]^{\top}$, $\mathcal{R}(U_t^*, \Vec{y}) = 0$. If $t = 0$, $\mathcal{R}(U_t^*, \Vec{y}) = 0$, which is proved in Theorem~\ref{th:sup_toy_extreme} when $\tau_s > \tau_c$. If $t \in (0, \Bar{t})$, as shown in Lemma~\ref{lemma:sup_eigprob}, we have 
$$
U^*_t=\left[\begin{array}{ccccc}
 {\lambda_1-1 \over 2t} & {\lambda_1-1 \over 2t} & {\tau_{s}(\lambda_1-1) \over 2(\lambda_1 -1 -\tau_{c})t} & {\tau_{s}(\lambda_1-1) \over 2(\lambda_1 -1 -\tau_{c})t} \\
 -1 & 1 & -1 & 1 \\
\end{array}\right]^{\top},
$$ where $\lambda_1 > 0$. $\Vec{\mu}_* = (U^{*\top}_tU^*_t)^{\dag}U^{*\top}_t\Vec{y} = [\frac{{\lambda_1-1 \over 2t}}{({\lambda_1-1 \over 2t})^2 + ({\tau_{s}(\lambda_1-1) \over 2(\lambda_1 -1 -\tau_{c})t})^2}, 0]^{\top}$, then: 

$$\mathcal{R}(U_t^*, \Vec{y}) = \frac{2\tau_s^2}{({\lambda_1 - 1 - \tau_c})^2 + \tau_s^2} = r(\lambda_1) \in (0,1). $$
Note that $\lambda_1$ is a value dependent on $t$, therefore $r(\lambda_1)$ can be represented as $r(t)$.

\end{proof}

\begin{lemma}  (Recap of Lemma~\ref{th:toy_harmful})
    If  $\tau_{s} < \tau_{c} < 1.5\tau_s$,  
$
    \mathcal{R}(U^*_3, \Vec{y}) = 1, \mathcal{R}(U^*_2, \Vec{y}) = 0. 
$
    \label{th:sup_toy_harmful}
\end{lemma} 
\begin{proof}
    When $\mathcal{X}_{l}^{\text{case 3}} \triangleq \{X_{\textcolor{gray}{\cube{0.5}}, \textcolor{gray}{c_3}}\}  (\text{gray cube})$, we have     
\begin{equation*}
T_3=\left[\begin{array}{ccccc}
\tau_1 & \tau_{s} & \tau_0 & \tau_{s} & \tau_0 \\
\tau_{s} & \tau_1 & \tau_{c} & \tau_{s} & \tau_0  \\
\tau_0 & \tau_{c} & \tau_1 & \tau_0 & \tau_{s}  \\
\tau_{s} & \tau_{s} & \tau_0 & \tau_1 & \tau_{c}  \\
\tau_0 & \tau_0 & \tau_{s} & \tau_{c} & \tau_1  \\
\end{array}\right],
\end{equation*}
    
    Follow the same proof in Lemma~\ref{lemma:sup_eigprob}, one can show that $$
U^*_3=\left[\begin{array}{ccccc}
 a_1 &  b_1 &  a_1 & b_1 \\
 a_2 &  b_2 &  a_2 & b_2 \\
\end{array}\right]^{\top},
$$
where $a_1,b_1$ are some positive real numbers, and $a_2,b_2$ has different signs. Note that $U^*_3$ forms the same linear span as 
$$ \frac{1}{2}\left[\begin{array}{cccc}
1 & 1 & 1 & 1\\ -1 & 1 & -1 & 1\\
\end{array}\right]^{\top}.$$ Therefore, we have $\mathcal{R}(U^*_3, \Vec{y}) = 1$ as proved in Theorem~\ref{th:sup_toy_extreme}.
\end{proof}

\section{Additional Details for Section~\ref{sec:theory_main}}

This section acts as an expanded version of Section~\ref{sec:theory_main}. We will first show in Section~\ref{sec:sup_no_approx} with the background and proof for Theorem~\ref{th:no_approx} with the original adjacency matrix $\Dot{A}$. Then we present the analysis based on the approximation matrix $\Bar{A}$ in Section~\ref{sec:sup_with_approx}. Finally, we show the formal proof of our main Theorem~\ref{th:main} in Section~\ref{sec:sup_main_proof}. The proof of Theorem~\ref{th:main} requires two important ingredients (Lemma~\ref{lemma:sup_error_bound_approx} and Lemma~\ref{lemma:sup_bound_kappa}) with proof deferred in Section~\ref{sec:sup_error_bound_approx} and Section~\ref{sec:sup_kappa_proof}   respectively.  

\subsection{Sufficient and Necessary Condition for Perfect Residual}

\label{sec:sup_no_approx}

We first present the formal analysis in Theorem~\ref{th:sup_no_approx} which is an extended version of  Theorem~\ref{th:no_approx} without approximation and we start with the recap of definitions. 

\textbf{Notations.}  Recall that $V^* \in \mathbb{R}^{N \times k}$ is defined as the top-$k$ singular vectors of $\Dot{A}$ and we split the eigen-matrix into two parts for labeled and unlabeled samples respectively:
$${V}^*=\left[\begin{array}{cc} 
{L}^* \in \mathbb{R}^{N_l \times k}\\ {U}^* \in \mathbb{R}^{N_u \times k}
\end{array}\right] = \left[\begin{array}{cccc}
{l}_1 & {l}_2 & \cdots &{l}_k \\ {u}_1 & {u}_2 & \cdots & {u}_k
\end{array}\right]$$
for labeled and unlabeled samples respectively. Then we let $V^{\flat} \in \mathbb{R}^{N \times (N-k)}$ be the remaining singular vectors of $\Dot{A}$ except top-$k$. Similarly, we split $V^{\flat}$ into two parts:
$${V}^{\flat}=\left[\begin{array}{cc} 
{L}^{\flat} \in \mathbb{R}^{N_l \times (N-k)} \\ {U}^{\flat} \in \mathbb{R}^{N_u \times (N-k)}
\end{array}\right] = \left[\begin{array}{cccc}
{l}_{k+1} & {l}_{k+2} & \cdots & {l}_{N} \\ {u}_{k+1} & {u}_{k+2} & \cdots & {u}_{N}
\end{array}\right].$$
We can also split the matrix $\Dot{A}$ at the $N_l$-th row and the $N_l$-th column and we obtain $A_{ll} \in \mathbb{R}^{N_l \times N_l},A_{ul} \in \mathbb{R}^{N_u \times N_l},A_{uu} \in \mathbb{R}^{N_u \times N_u}$ with
$$\Dot{A} = \left[\begin{array}{cc} 
A_{ll} & A^{\top}_{ul} \\ 
 A_{ul} & A_{uu}
\end{array}\right].$$

\begin{theorem}
    (\textbf{No approximation})  
    Denote the projection matrix $\mathsf{P}_{L^{\flat}} = L^{\flat\top}(L^{\flat}L^{\flat\top})^{\dag}L^{\flat}$, where $^{\dag}$ denotes the Moore-Penrose inverse. For any labeling vector $\Vec{y} \in \{0,1\}^{N_u}$, we have
    \begin{equation}
        \mathcal{R}(U^*, \Vec{y}) \leq \|(I-\mathsf{P}_{L^{\flat}}) U^{\flat\top} \Vec{y}\|^2_2.
    \label{eq:sup_R_bound_no_approx}
    \end{equation}
    The sufficient and necessary condition for $\mathcal{R}(U^*, \Vec{y}) = 0$ is $\Vec{\omega} \in \mathbb{R}^{N_l}$ such that   
    \begin{equation}
        \forall i = k+1, \ldots, N, \langle \Vec{y}^{\top} (\sigma_i I - A_{uu})^{\dag}A_{ul}, l_i \rangle =  \langle\Vec{\omega}, l_i \rangle
    \label{eq:sup_complex_condition}
    \end{equation}       
    where $\sigma_i$ is the $i$-th largest eigenvalue of $\Dot{A}$. 
    \label{th:sup_no_approx}
\end{theorem}

\begin{proof}

Define $\Vec{y}' = [\Vec{\zeta}^{\top} , \Vec{y}^{\top}]^{\top}$ as an extended labeling vector, where $\Vec{\zeta} \in \mathbb{R}^{N_l}$ can be a ``placeholder'' vector with any values. We have 
    \begin{align*}
        \mathcal{R}\left(U^*, \vec{y}\right) &= \min_{\Vec{\mu} \in \mathbb{R}^{k }} \|\Vec{y} - U^* \Vec{\mu} \|^2_2  \\
        & = \min_{\Vec{\mu} \in \mathbb{R}^{k }, \Vec{\zeta}\in \mathbb{R}^{N_l}} \|\Vec{y}' - V^* \Vec{\mu} \|^2_2 \\ 
        & =  \min_{\Vec{\zeta}\in \mathbb{R}^{N_l}} \|\Vec{y}' - V^*V^{*\top} \Vec{y}'\|^2_2 \\
        &=  \min_{\Vec{\zeta}\in \mathbb{R}^{N_l}} \|V^{\flat\top} \Vec{y}'\|^2_2 \\
        & = \min_{\Vec{\zeta}\in \mathbb{R}^{N_l}} \| L^{\flat\top} \Vec{\zeta} + U^{\flat\top} \Vec{y}\|^2_2 \\
         & = \|(I-L^{\flat\top}(L^{\flat}L^{\flat\top})^{\dag}L^{\flat}) U^{\flat\top} \Vec{y}\|^2_2.
    \end{align*}

The sufficient and necessary condition for $\mathcal{R}(U^*, \Vec{y}) = 0$ is: $$\exists \vec{\omega} \in \mathbb{R}^{N_l}, \forall i = k+1, \ldots, N,  u_i^{\top} \Vec{y} = l_i^{\top}\Vec{\omega}.$$

We then look into the relationship between $l_i$ and $u_i$.
Since 
$$\left[\begin{array}{cc} 
A_{ll} & A^{\top}_{ul} \\ 
 A_{ul} & A_{uu} 
\end{array}\right]\left[\begin{array}{c} 
 l_i \\ 
 u_i 
\end{array}\right] = \sigma_i \left[\begin{array}{c} 
 l_i \\ 
 u_i 
\end{array}\right],$$
we have the following results: 
$$u_i = (\sigma_i I - A_{uu})^{\dag}A_{ul}l_i.$$

So the sufficient and necessary condition becomes: there exists $\Vec{\omega} \in \mathbb{R}^{N_l}$ such that   
    \begin{equation}
        \forall i = k+1, \ldots, N,\langle \Vec{y}^{\top} (\sigma_i I - A_{uu})^{\dag}A_{ul}, l_i \rangle =  \langle\Vec{\omega}, l_i \rangle,
    \label{eq:complex_condition}
    \end{equation}       
    where $\sigma_i$ is the $i$-th largest singular value of $\Dot{A}$. 
\end{proof}

\textbf{Interpretation of Theorem~\ref{th:sup_no_approx}.} The bound of residual in Ineq.~\eqref{eq:R_bound_no_approx} composed of two projections:  $U^{\flat\top}$ and $(I-\mathsf{P}_{L^{\flat}})$. If we only consider $\|U^{\flat\top}\Vec{y}\|^2_2$, it is equivalent to $\Vec{y}^{\top} (I - U^{*}U^{*\top}) \Vec{y}$ which indicates the information in $\Vec{y}$ that is not covered by the learned representation $U^{*}$. Then multiplying the second projection matrix $(I-\mathsf{P}_{L^{\flat}})$ further reduces the residual by considering the information from labeled data, since $\mathsf{P}_{L^{\flat}}$ is a projection matrix that projects a vector to the linear span of $L^{\flat}$. In the extreme case, when $U^{\flat\top} \Vec{y}$ fully lies in the linear span of $L^{\flat}$, the residual $\mathcal{R}(U^*, \Vec{y})$ becomes 0. 
To provide further insights about Eq.~\eqref{eq:sup_complex_condition}, we analyze in a simplified setting by approximating $\Dot{A}$ in the next section. 

\subsection{Analysis with Approximation} 
\label{sec:sup_with_approx}
In Theorem~\ref{th:sup_no_approx}, we put an analysis on how $L^{\flat}$ can influence the residual function. However, $L^{\flat}$ is a matrix with $N_l$ rows, so it is hard to quantitatively understand the effect of $N_l$ labeled samples individually. We resort to viewing the labeled samples as a whole. Our idea is motivated by the Stochastic Block Model (SBM)~\cite{holland1983stochastic} model, which analyzes the probability between different communities instead of individual values. In our case, we aim to analyze the probability vector $\eta_u \in \mathbb{R}^{N_u}$ denoting the chance of each unlabeled data point having the same augmentation view as one of the samples from the known class. The relationship between $\eta_u$ and $A_{uu}$ is then of our interest. 
Specifically, we define $\Bar{A}$ with values at $(i, j )$ be the following: 

\begin{equation}
    \Bar{A}_{x_ix_j} = \left\{\begin{array}{cc}             \Dot{A}_{x_i x_j} & \text{if } x_i \in \mathcal{X}_u, x_j \in \mathcal{X}_u,   \\
        \mathbb{E}_{x' \in \mathcal{X}_l} \Dot{A}_{x_i x'} & \text{if } x_i \in \mathcal{X}_u, x_j \in \mathcal{X}_l,   \\
        \mathbb{E}_{x' \in \mathcal{X}_l} \Dot{A}_{x' x_j } & \text{if } x_i \in \mathcal{X}_l, x_j \in \mathcal{X}_u,   \\
        \mathbb{E}_{x', x'' \in \mathcal{X}_l} \Dot{A}_{x' x''} & \text{if } x_i \in \mathcal{X}_l, x_j \in \mathcal{X}_l.   \\
    \end{array} \right.
    \label{eq:sup_abar_def}
\end{equation}

The probability is estimated by taking the average. It is equivalent to multiplying matrix $P$ and $P^{\top}$ on left and right side, where $P \in \mathbb{R}^{N \times N}$ is given by: 
$$P=\left[\begin{array}{cc} 
 \frac{1}{N_l}\mathbf{1}_{N_l \times N_l}  & \mathbf{0}_{N_l \times N_u} \\
 \mathbf{0}_{N_u \times N_l} &  I_{N_u} 
\end{array}\right],$$

where $\mathbf{1}_{n\times m}$ and $\mathbf{0}_{n\times m}$ represent matrix filled with 1 and 0 respectively with shape $n\times m$. Then we can write $\bar{A} \in \mathbb{R}^{N \times N}$, the approximated version of $A$, as follows: 
$$\bar{A}= PAP^{\top} = \left[\begin{array}{cc} 
\eta_l\mathbf{1}_{N_l \times N_l} & \mathbf{1}_{N_l \times 1}\Vec{\eta}_u^{\top} \\ 
\Vec{\eta}_u \mathbf{1}_{1 \times N_l} & A_{uu},
\end{array}\right],$$
where $\eta_l \in \mathbb{R}$ and $\Vec{\eta}_u \in \mathbb{R}^{N_u\times 1}$. Our analysis can then focus on how $\eta_u$ influences the representation space learned by $A_{uu}$. 
Similar to Section~\ref{sec:sup_no_approx}, we define the top-$k$ and the remainder singular vectors with corresponding splits as :
$$\Bar{V}^*=\left[\begin{array}{cc} 
\Bar{L}^* \\ \Bar{U}^*
\end{array}\right] = \left[\begin{array}{cccc}
\Bar{l}_1 & \Bar{l}_2 & \cdots &\Bar{l}_k \\ \Bar{u}_1 & \Bar{u}_2 & \cdots & \Bar{u}_k
\end{array}\right],$$
$$\Bar{V}^{\flat}=\left[\begin{array}{cc} 
\Bar{L}^{\flat} \\ \Bar{U}^{\flat}
\end{array}\right] = \left[\begin{array}{cccc}
\Bar{l}_{k+1} & \Bar{l}_{k+2} & \cdots & \Bar{l}_{N} \\ \Bar{u}_{k+1} & \Bar{u}_{k+2} & \cdots & \Bar{u}_{N}
\end{array}\right].$$
Note that due to the special structure of $\Bar{A}$ with $N_l$ duplicated rows and columns, the eigenvector $\Bar{V}$ has a special structure as we demonstrate in the next Lemma~\ref{lemma:sup_l_form}. We defer the proof to Section~\ref{sec:sup_l_form_proof}.

\begin{lemma}
   Since $A_{uu}$ is symmetric and has large diagonal values, we assume $A_{uu}$ is a positive semi-definite matrix. $\Bar{L}^*$ is stacked by the same row such that 
$ \Bar{L}^* = \frac{\mathbf{1}_{N_l \times 1}}{N_l} \Bar{\mathfrak{l}}^{*\top},$ where $\Bar{\mathfrak{l}}^{*} \in \mathbb{R}^{k}$ and that $\Bar{L}^\flat$ has the following form: 
$$\Bar{L}^\flat = \left[\begin{array}{cccc}  \frac{\mathbf{1}_{N_l\times 1}}{N_l} \Bar{\mathfrak{l}}^{\prime\top} & \Bar{l}_{N - \Theta + 1} & ... & \Bar{l}_{N}
\end{array}\right],$$ where $\Theta$ is the rank of the null space for $A_{uu} - \frac{\eta_u\eta_u^{\top}}{\eta_l}$,  
 $\Bar{\mathfrak{l}}^{\prime} \in \mathcal{R}^{N-k-\Theta}$ with non-zero values, and $\Bar{l}_{N - \Theta + 1}, ..., \Bar{l}_{N}$ are all perpendicular to $\mathbf{1}_{N_l}$. 
\label{lemma:sup_l_form}
\end{lemma}

By property in Lemma~\ref{lemma:sup_l_form}, we define: 
\begin{equation}
    \Bar{\mathfrak{l}}^{\flat} \triangleq  \Bar{L}^{\flat\top} \mathbf{1}_{ N_l \times 1} = \left[\begin{array} {cccc}\Bar{\mathfrak{l}}^{\prime\top} &  0 &  ... & 0\end{array}\right]^\top \in \mathbb{R}^{N-k}.
\end{equation}

\begin{definition}
    To ease the notation, we let $\mathcal{I} \triangleq \{k+1, k+2, ..., N-\Theta\}$ and we mainly discuss $i \in \mathcal{I}$.
\end{definition}
These definitions facilitate the presentation of the following Theorem~\ref{th:sup_with_approx}.
\begin{theorem}
    (\textbf{With approximation})  Denote $\mathfrak{T}(\Vec{y}) = \frac{\|\Bar{U}^{\flat\top} \Vec{y}\|_2}{\|\Vec{y}\|_2}$ and  $\kappa(\Vec{y}) = \cos(\Bar{U}^{\flat\top} \Vec{y},\Bar{\mathfrak{l}}^{\flat})$, where $\cos$ measures the cosine distance between two vectors. Let  $\sigma_i$ as the $i$-th largest eigenvalue of $\Dot{A}$ and $\Bar{\sigma}_i$ is for $\Bar{A}$. 
    For a labeling vector $\Vec{y} \in \{0,1\}^{N_u}$, we have
    \begin{equation}
         \mathcal{R}(\Bar{U}^*, \Vec{y}) = \frac{N_u}{|\mathcal{Y}_u|} (1-\kappa(\Vec{y})^2) \mathfrak{T}(\Vec{y})^2. 
    \end{equation}
    If the ignorance degree $\mathfrak{T}(\Vec{y})$ is non-zero, the sufficient and necessary condition for $\mathcal{R}(\Bar{U}^*, \Vec{y}) = 0$: there exists $\omega \in \mathbb{R}$ such that  
    \begin{equation}
        \forall i \in \mathcal{I},  \Vec{y}^{\top} (\Bar{\sigma}_i I - A_{uu})^{\dag}\Vec{\eta}_u= \omega.
    \label{eq:easy_condition}
    \end{equation}    
    \label{th:sup_with_approx}
\end{theorem}

\begin{proof}

Define $\Vec{y}' = [\zeta\mathbf{1}_{1\times N_l} , \Vec{y}^{\top}]^{\top}$ as an extended labeling vector where $\zeta$  is any real number. We have 
    \begin{align*}
        \mathcal{R}\left(\Bar{U}^*, \vec{y}\right) &= \min_{\Vec{\mu} \in \mathbb{R}^{k }} \|\Vec{y} - \Bar{U}^* \Vec{\mu} \|^2_2  \\
        &= \min_{\Vec{\mu} \in \mathbb{R}^{k }, \zeta\in \mathbb{R}} \{ \|\Vec{y} - \Bar{U}^* \Vec{\mu} \|^2_2 + \|(\zeta - \Bar{\mathfrak{l}}^{*\top}\Vec{\mu})\mathbf{1}_{1\times N_l}\|_2^2 \} \\
        &= \min_{\Vec{\mu} \in \mathbb{R}^{k }, \zeta \in \mathbb{R}} \|\Vec{y}' - \Bar{V}^* \Vec{\mu} \|^2_2 \\ 
        & =  \min_{\zeta \in \mathbb{R}} \|\Vec{y}' - \Bar{V}^*\Bar{V}^{*\top} \Vec{y}'\|^2_2 \\
        &=  \min_{\zeta \in \mathbb{R}} \|\Bar{V}^{\flat\top} \Vec{y}'\|^2_2 \\
        & = \min_{\zeta \in \mathbb{R}} \|  \zeta \Bar{L}^{\flat\top} \mathbf{1}_{N_l\times 1}+ \Bar{U}^{\flat\top} \Vec{y}\|^2_2 \\
        &= \min_{\zeta \in \mathbb{R}} \|  \zeta \Bar{\mathfrak{l}}^{\flat}+ \Bar{U}^{\flat\top} \Vec{y}\|^2_2 \\
         & = \|(I - \frac{\Bar{\mathfrak{l}}^{\flat}\Bar{\mathfrak{l}}^{\flat\top}}{\|\Bar{\mathfrak{l}}^{\flat}\|^2_2}) \Bar{U}^{\flat\top} \Vec{y}\|^2_2 \\
         & = (1-\kappa(\Vec{y})^2) \|\Bar{U}^{\flat\top} \Vec{y}\|^2_2 \\
         & = \frac{N_u}{|\mathcal{Y}_u|} (1-\kappa(\Vec{y})^2) \mathfrak{T}(\Vec{y})^2.
    \end{align*}

We then look into the components of $\Bar{\mathfrak{l}}^{\flat}$ and $\Bar{U}^{\flat}$.
According to Lemma~\ref{lemma:sup_l_form}, when $i > N - \Theta$, we have:
\begin{equation}\Bar{\mathfrak{l}}^{\flat} = \left[\begin{array} {cccc}\Bar{\mathfrak{l}}^{\prime\top} &  0 &  ... & 0\end{array}\right]^\top = [\begin{array}{ccccccc}(\Bar{\mathfrak{l}}^{\flat})_{k+1} & (\Bar{\mathfrak{l}}^{\flat})_{k+2} & \cdots & (\Bar{\mathfrak{l}}^{\flat})_{N - \Theta} & 0  \cdots & 0 \end{array}].
\label{eq:sup_l_def}
\end{equation}

And the sufficient and necessary condition for $\mathcal{R}(\Bar{U}^*, \Vec{y})$ to be minimized by $\Bar{\mathfrak{l}}^{\flat}$ is: \begin{equation}
\exists \omega \in \mathbb{R}, \forall i \in \mathcal{I}, \Bar{u}_{i}^{\flat\top} \Vec{y} = \omega (\Bar{\mathfrak{l}}^{\flat})_{i}.
\label{eq:sup_uywl}
\end{equation}

Note that for $ i \in \mathcal{I}$,
$$\left[\begin{array}{cc} 
\eta_l\mathbf{1}_{N_l \times N_l} & \mathbf{1}_{N_l \times 1}\Vec{\eta}_u^{\top} \\ 
\Vec{\eta}_u \mathbf{1}_{1 \times N_l} & A_{uu}
\end{array}\right]\left[\begin{array}{c} 
 \Bar{l}_i \\ 
 \Bar{u}_i 
\end{array}\right] = \Bar{\sigma}_i \left[\begin{array}{c} 
 \Bar{l}_i \\ 
 \Bar{u}_i 
\end{array}\right].$$
Also since $(\Bar{\mathfrak{l}}^{\flat})_i = \mathbf{1}_{1 \times N_l} \Bar{l}_i \in \mathbb{R}$,
we have the following results: 
$$\Bar{u}_i = (\Bar{\sigma}_i I - A_{uu})^{\dag}\Vec{\eta}_u (\Bar{\mathfrak{l}}^{\flat})_i.$$

Thus, the sufficient and necessary condition~\eqref{eq:sup_uywl} becomes: there exists $\omega \in \mathbb{R}$ such that   
    \begin{equation}
        \forall i \in \mathcal{I}, \Vec{y}^{\top}(\Bar{\sigma}_i I - A_{uu})^{\dag}\Vec{\eta}_u= \omega.
    \end{equation}     
\end{proof}

\subsubsection{Proof of Lemma~\ref{lemma:sup_l_form}} \label{sec:sup_l_form_proof}
\begin{proof}
To understand the structure of $\Bar{U}$ and $\Bar{L}$, we consider the eigenvalue problem:
    $$\left[\begin{array}{cc} 
\eta_l\mathbf{1}_{N_l \times N_l} & \mathbf{1}_{N_l \times 1}\Vec{\eta}_u^{\top} \\ 
\Vec{\eta}_u \mathbf{1}_{1 \times N_l} & A_{uu}
\end{array}\right]\left[\begin{array}{c} 
 \Bar{l}_i \\ 
 \Bar{u}_i 
\end{array}\right] = \Bar{\sigma}_i \left[\begin{array}{c} 
 \Bar{l}_i \\ 
 \Bar{u}_i 
\end{array}\right].$$
In the non-trivial case, $\eta_l \neq 0, \Vec{\eta}_u \neq \mathbf{0}_{N_l}$ , we have the following two equations: 
\begin{align*}
    \eta_l\mathbf{1}_{N_l \times 1} \mathbf{1}_{1 \times N_l}  \Bar{l}_i + \mathbf{1}_{N_l \times 1}\Vec{\eta}_u^{\top}  \Bar{u}_i &= \Bar{\sigma}_i \Bar{l}_i \\
    (\Bar{\sigma}_i I - A_{uu})\Bar{u}_i  &= \Vec{\eta}_u \mathbf{1}_{1 \times N_l} \Bar{l}_i.
\end{align*}

\textbf{(Case 1)} When $\Bar{\sigma}_i \neq 0$, then $\Bar{l}_i$ has $N_l$ duplicated scalar values $\frac{ \Vec{\eta}_u^{\top}  \Bar{u}_i}{\Bar{\sigma}_i - N_l\eta_l}$ for the first equation to satisfy.  

\textbf{(Case 2)} When $\Bar{\sigma}_i = 0$, then by combing the two equations, we have: 
$$A_{uu}\Bar{u}_i = \frac{\Vec{\eta}_u\Vec{\eta}^\top_u}{\eta_l} \Bar{u}_i. $$ 
If $A_{uu} - \frac{\Vec{\eta}_u\Vec{\eta}^\top_u}{\eta_l}$ is a full rank matrix, then $\Bar{u}_i = \mathbf{0}_{N_u}$, and by the first equation $\mathbf{1}_{1 \times N_l}\Bar{l}_i = 0$. If $A_{uu} - \frac{\Vec{\eta}_u\Vec{\eta}^\top_u}{\eta_l}$ is a deficiency matrix and $\text{rank}(A_{uu} - \frac{\Vec{\eta}_u\Vec{\eta}^\top_u}{\eta_l}) \ge \text{rank}(A_{uu})$\footnote{When $\text{rank}(A_{uu} - \frac{\Vec{\eta}_u\Vec{\eta}^\top_u}{\eta_l}) < \text{rank}(A_{uu})$, it means that $\eta_u$ happens to cancel out one of the direction in $A_{uu}$. Such an event has zero probability almost sure in reality. We do not consider this case in our proof. }, then $\Bar{u}_i$ lies in the null space formed by $\Vec{\eta}_u$ and $A_{uu}$ jointly, then $\Vec{\eta}^{\top}_u \Bar{u}_i = 0$, we still have $\mathbf{1}_{1 \times N_l}\Bar{l}_i = 0$. 

Therefore when $i \in \{1,\dots, k\}$, $\Bar{\sigma}$ is non-zero values, so that $\Bar{L}^*$ is stacked by the same row such that 
$ \Bar{L}^* = \frac{\mathbf{1}_{N_l \times 1}}{N_l} \Bar{\mathfrak{l}}^{*\top},$ where $\Bar{\mathfrak{l}}^{*} \in \mathbb{R}^{k}$. 
For $i \in \{k+1,\dots, N\}$,
$\Bar{L}^\flat$ has the following form: 
$$\Bar{L}^\flat = \left[\begin{array}{cccc}  \frac{\mathbf{1}_{N_l\times 1}}{N_l} \Bar{\mathfrak{l}}^{\prime\top} & \Bar{l}_{N - \Theta + 1} & ... & \Bar{l}_{N}
\end{array}\right],$$ where $\Theta$ is the rank of the null space for $A_{uu} - \frac{\eta_u\eta_u^{\top}}{\eta_l}$,  
 $\Bar{\mathfrak{l}}^{\prime} \in \mathcal{R}^{N-k-\Theta},$ and $\Bar{l}_{N - \Theta + 1}, ..., \Bar{l}_{N}$ are all perpendicular to $\mathbf{1}_{N_l}$.
\end{proof}

\subsection{Proof for the Main Theorem~\ref{th:main}}
\label{sec:sup_main_proof}
In this section, we provide the main proof of Theorem~\ref{th:main}. For reader's convenience, we provide the recap version in Theorem~\ref{th:sup_main} by omitting the definition claim, where the detailed definition of $A_{ul}, A_{ll}, q_i, \Bar{U}^{\flat\top}, \Bar{\mathfrak{l}}^{\flat}, \Vec{\eta}_u$ is in Section~\ref{sec:sup_with_approx}. 

The proof of Theorem~\ref{th:main} consists of four steps. Firstly, $\mathcal{E}(f)$ is bounded by $\mathcal{R}(U^*)$ as we show in Lemma~\ref{lemma:cls_bound}. Secondly, the residual $\mathcal{R}\left(U^*, \vec{y}\right)$ of the original representation can be approximated by the residual $\mathcal{R}\left(\Bar{U}^*, \vec{y}\right)$ analyzed in Section~\ref{sec:sup_with_approx}.  Thirdly, the approximation error bound is in the order of $\frac{\|\Dot{A} - \bar{A}\|_2}{\sigma_{k} - \sigma_{k+1}}  $ as shown in Section~\ref{sec:sup_error_bound_approx}. Finally, we show that the coverage measurement $\kappa(\Vec{y})$ can be lower bounded in Section~\ref{sec:sup_kappa_proof}.

\begin{theorem} (Recap of Theorem~\ref{th:main}) Based on the assumptions made in Lemma~\ref{lemma:sup_error_bound_approx}, Lemma~\ref{lemma:sup_w_bound} and Lemma~\ref{lemma:sup_bound_kappa}.
The linear probing error is bounded by:
    \begin{equation}
        \mathcal{E}(f) \lesssim \frac{2N_u}{|\mathcal{Y}_u|}\left(\sum_i^{|\mathcal{Y}_u|} \mathfrak{T}(\Vec{y}_i)(1-\kappa(\Vec{y}_i)^2) + \frac{\|\Dot{A} - \bar{A}\|_2}{\sigma_{k} - \sigma_{k+1}} \right),
    \end{equation}
    where for single labeling vector $\Vec{y}$, $$\kappa(\Vec{y}) = \cos(\Bar{U}^{\flat\top} \Vec{y},\Bar{\mathfrak{l}}^{\flat}) 
    \gtrsim \min_{i > k, j > k} \frac{2\sqrt{\frac{\Vec{y}^{\top}q_i}{\Vec{\eta}_u^{\top}q_i}\frac{\Vec{y}^{\top}q_j}{\Vec{\eta}_u^{\top}q_j}}}{\frac{\Vec{y}^{\top}q_i}{\Vec{\eta}_u^{\top}q_i}+\frac{\Vec{y}^{\top}q_j}{\Vec{\eta}_u^{\top}q_j}}.$$
    \label{th:sup_main}
\end{theorem}

\begin{proof}
    According to Lemma~\ref{lemma:cls_bound}, we have 
        $$\mathcal{E}(f) \leq 2\mathcal{R}(U^*) = 2\sum_{i \in \mathcal{Y}_u} \mathcal{R}(U^*, \Vec{y}_i),$$
    where we can view each $\Vec{y}_i$ separately. For simplicity, we use $\Vec{y}$ in the following proof. 
    As show in Section~\ref{sec:sup_with_approx}, $\mathcal{R}(U^*, \Vec{y})$ can be approximately estimated by $\mathcal{R}(\Bar{U}^*, \Vec{y}_i) = (1-\kappa(\Vec{y})^2) \|\Bar{U}^{\flat\top} \Vec{y}_i\|^2_2 = \mathfrak{T}(\Vec{y}_i)(1-\kappa(\Vec{y})^2) \|\Vec{y}_i\|^2_2$.  Such approximation bound is given by 
$$\mathcal{R}(U^*, \Vec{y}) \lesssim  \mathcal{R}(\Bar{U}^*, \Vec{y}) + \frac{2\|\Dot{A} - \bar{A}\|_2}{\sigma_{k} - \sigma_{k+1}}   \|\Vec{y}\|_2^2,$$
    as shown in Lemma~\ref{lemma:sup_error_bound_approx} in Section~\ref{sec:sup_error_bound_approx}. Putting things together, we have 
    \begin{equation*}
        \mathcal{E}(f) \lesssim 2\sum_i^{|\mathcal{Y}_u|} \mathfrak{T}(\Vec{y}_i)(1-\kappa(\Vec{y})^2) \|\Vec{y}_i\|^2_2  + \frac{2\|\Dot{A} - \bar{A}\|_2}{\sigma_{k} - \sigma_{k+1}}  \|\Vec{y}_i\|_2^2.
    \end{equation*}
    If the sample size in the novel class is balanced, we have $\|\Vec{y}\|^2_2 = \frac{N_u}{|\mathcal{Y}_u|}$, we have: 
    \begin{equation*}
        \mathcal{E}(f) \lesssim \frac{2N_u}{|\mathcal{Y}_u|}\left( \sum_i^{|\mathcal{Y}_u|} \mathfrak{T}(\Vec{y}_i)(1-\kappa(\Vec{y})^2) + \frac{\|\Dot{A} - \bar{A}\|_2}{\sigma_{k} - \sigma_{k+1}} \right), 
    \end{equation*}
    Finally, the lower bound of $\kappa$ is given by Lemma~\ref{lemma:sup_bound_kappa} and proved in Section~\ref{sec:sup_kappa_proof}. 
\end{proof}

\subsection{Error Bound by Approximation }
\label{sec:sup_error_bound_approx}
We see in Section~\ref{sec:sup_with_approx} that we use the approximated version $\Bar{U}^*$ instead of the actual feature representation $U^*$, which creates a gap. In this section, we will present a formal analysis on  the gap between the induced residuals $\mathcal{R}(U^*, \Vec{y})$ and $\mathcal{R}(\Bar{U}^*, \Vec{y})$.

\begin{lemma}
        When $\|\Dot{A} - \bar{A}\|_2 < \frac{1}{2}({\sigma}_{k} - {\sigma}_{k+1})$ and $|\mathcal{Y}_u| \triangleq \mathbb{E}_{i\in \mathcal{I}} (1 - \|\Bar{u}_i\|^2_2)$ is a non-zero value\footnote{Note that $|\mathcal{Y}_u| = 0$ happens in an extreme case that $\forall i \in \mathcal{I}, \|\Bar{l}_i\|_2^2 = 0 $ which means the extra knowledge is purely irrelevant to the feature representation. Specifically, this could happen when $A_{ul}$ (defined in Section~\ref{sec:sup_no_approx}) is a zero matrix.}, we have
        $$\mathcal{R}(U^*, \Vec{y}) \lesssim  \mathcal{R}(\Bar{U}^*, \Vec{y}) + 2\frac{\|\Dot{A} - \bar{A}\|_2}{\sigma_{k} - \sigma_{k+1}}  \|\Vec{y}\|_2^2.$$
\label{lemma:sup_error_bound_approx}
\end{lemma}
\begin{proof}
    Recall that $\Vec{y}' = [\zeta\mathbf{1}_{1\times N_l} , \Vec{y}^{\top}]^{\top}$ is  an extended labeling vector where $\zeta$ is any real number defined in the proof of Theorem~\ref{th:sup_with_approx}. We let $\zeta^* = \arg \min_{\zeta \in \mathbb{R}} \|\Bar{V}^{\flat\top} \Vec{y}'\|^2_2 $ so that $\Bar{y}^* = [\zeta^*\mathbf{1}_{1\times N_l} , \Vec{y}^{\top}]$. 
We then define $\delta \triangleq \min\{\sigma_{k} - \bar{\sigma}_{k+1},  \bar{\sigma}_{k} - \sigma_{k+1} \}$,
\begin{align*}
    \mathcal{R}(U^*, \Vec{y}) = &  \min_{\zeta \in \mathbb{R}} \|{V}^{\flat\top} \Vec{y}'\|^2_2  \\
    = & \min_{\zeta \in \mathbb{R}} {\Vec{y}^{'\top}} {V}^{\flat} {V}^{\flat\top} \Vec{y}' \\
    = &  \min_{\zeta \in \mathbb{R}} ({\Vec{y}^{'\top}} \Bar{V}^{\flat} \Bar{V}^{\flat\top} \Vec{y}' + {\Vec{y}^{'\top}} {V}^{\flat} {V}^{\flat\top} \Vec{y}' - {\Vec{y}^{'\top}} \Bar{V}^{\flat} \Bar{V}^{\flat\top} \Vec{y}')\\
    \le & \mathcal{R}(\Bar{U}^*, \Vec{y}) +|{\Bar{y}^{*\top}} ({V}^{\flat} {V}^{\flat\top} - \Bar{V}^{\flat} \Bar{V}^{\flat\top}) \Bar{y}^*|\\
    \le & \mathcal{R}(\Bar{U}^*, \Vec{y}) +\|{V}^{\flat} {V}^{\flat\top} - \Bar{V}^{\flat} \Bar{V}^{\flat\top}\| \|\Bar{y}^*\|^2_2\\
    = & \mathcal{R}(\Bar{U}^*, \Vec{y}) +\|{V}^{\flat\top}\Bar{V}^{*}\| \|\Bar{y}^*\|^2_2\\
    \le & \mathcal{R}(\Bar{U}^*, \Vec{y}) +{\|\Dot{A} - \bar{A}\|_2\over \delta} \|\Bar{y}^*\|^2_2\\
    \le & \mathcal{R}(\Bar{U}^*, \Vec{y}) +{2\|\Dot{A} - \bar{A}\|_2\over \sigma_{k} - \sigma_{k+1} } \|\Bar{y}^*\|^2_2,
\end{align*}
where the second last inequality is from Davis-Kahan theorem on subspace distance $\|{V}^{\flat} {V}^{\flat\top} - \Bar{V}^{\flat} \Bar{V}^{\flat\top}\| = \|{V}^{\flat\top}\Bar{V}^{*}\| = \|\Bar{V}^{\flat\top}{V}^{*}\|$, and the last inequality is from Weyl’s inequality so that $\delta \ge ({\sigma}_{k} - {\sigma}_{k+1}) - \|\Dot{A} - \bar{A}\|_2 \ge \frac{1}{2}({\sigma}_{k} - {\sigma}_{k+1})$.

We then investigate the magnitude order of $ \|\Bar{y}^*\|^2_2$. Note that $\|\Bar{y}^*\|^2_2 = \|\Vec{y}\|_2^2 + N_l (\zeta^*) ^2$ and $\zeta^* = {{\Bar{\mathfrak{l}}^{\flat\top}} \Bar{U}^{\flat\top} \Vec{y} \over  {\|\Bar{\mathfrak{l}}^{\flat}\|^2_2}}$ according to the proof of Theorem~\ref{th:sup_with_approx}. Then, 
\begin{align*}
    \|\Bar{y}^*\|^2_2  &= \|\Vec{y}\|_2^2 + {N_l ({\Bar{\mathfrak{l}}^{\flat\top}} \Bar{U}^{\flat\top} \Vec{y})^2 \over  {\|\Bar{\mathfrak{l}}^{\flat}\|^4_2}}\\
    &= \|\Vec{y}\|_2^2 + \frac{N_l \kappa(\Vec{y})^2 \|\Bar{U}^{\flat\top} \Vec{y}\|^2_2}{\|\Bar{\mathfrak{l}}^{\flat}\|^2_2}  \\
    &= \|\Vec{y}\|_2^2\left(1 + \frac{N_l \kappa(\Vec{y})^2 \mathfrak{T}(\Vec{y})^2}{\|\Bar{\mathfrak{l}}^{\flat}\|^2_2}\right)  \\
    &= \|\Vec{y}\|_2^2\left(1 + \frac{\kappa(\Vec{y})^2 \mathfrak{T}(\Vec{y})^2}{\sum_{i=k+1}^{N - \Theta} (1 - \|\Bar{u}_i\|^2_2)}\right), 
\end{align*}
where the last equation is given by  Lemma~\ref{lemma:sup_l_form} when $i > N - \Theta$, $(\Bar{\mathfrak{l}}^{\flat})_i = 0$ and also by the fact that when $i \in \mathcal{I}$, $1 -\|\Bar{u}_i\|_2^2 = \|\Bar{l}_i\|_2^2 = N_l (\frac{(\Bar{\mathfrak{l}}^{\flat})_i}{N_l}) ^ 2 = (\Bar{\mathfrak{l}}^{\flat})_i^2 / N_l$. Then by the assumption that $|\mathcal{Y}_u|$ is non-zero, we have 
\begin{align*}
    \|\Bar{y}^*\|^2_2  &= \|\Vec{y}\|_2^2(1 + \frac{\kappa(\Vec{y})^2 \mathfrak{T}(\Vec{y})^2}{(N - \Theta - k)|\mathcal{Y}_u|}) \lesssim \|\Vec{y}\|_2^2(1 + O(\frac{1}{N})).
\end{align*}
By plugging back $\|\Bar{y}^*\|^2_2$, we have         
$$\mathcal{R}(U^*, \Vec{y}) \lesssim  \mathcal{R}(\Bar{U}^*, \Vec{y}) + \frac{2\|\Dot{A} - \bar{A}\|_2}{\sigma_{k} - \sigma_{k+1}}   \|\Vec{y}\|_2^2.$$
\end{proof}

\subsection{Analysis on the Coverage Measurement $\kappa(\Vec{y})$}
\label{sec:sup_kappa_proof}
So far we have shown in Theorem~\ref{th:sup_with_approx} that the sufficient and necessary condition for a zero residual is when the coverage measurement $\kappa(\Vec{y}) = \cos(\Bar{U}^{\flat\top} \Vec{y},\Bar{\mathfrak{l}}^{\flat})$ equals to one. %
In this section, we provide a deeper analysis on $\kappa(\Vec{y})$ in a less restrictive case.

Recall that we have proved in Theorem~\ref{th:sup_with_approx} that the sufficient and necessary condition for $\kappa(\Vec{y}) = 1$ is: 
    \begin{equation}
        \exists \omega \in \mathbb{R}, \forall i \in \mathcal{I}, \Vec{y}^{\top}(\Bar{\sigma}_i I - A_{uu})^{\dag}\Vec{\eta}_u= \omega.
    \end{equation}     
In a general case, we consider $\omega_i$ which is variant on $i$: 
$$\omega_i  \triangleq \Vec{y}^{\top}(\Bar{\sigma}_i I - A_{uu})^{\dag}\Vec{\eta}_u.$$ 

Our discussion on $\kappa(\Vec{y})$ is based on the following definitions:

\begin{definition}
Let $q_j$ and $d_j$ as the $j$-th eigenvector/eigenvalue of $A_{uu}$.   Then we define $\Tilde{\mathbf{y}}_j \triangleq \Vec{y}^\top q_j$ and $\Tilde{\boldsymbol{\eta}}_j \triangleq \Vec{\eta}_u^\top q_j$. 
\label{def:sup_y_eta}
\end{definition}

Before showing the bound on $\kappa(\Vec{y})$, we first show the following Lemma~\ref{lemma:sup_closed_form} and Lemma~\ref{lemma:sup_w_bound} which is the important ingredient needed to derive the lower bound of $\kappa(\Vec{y})$. We defer the proof to Section~\ref{sec:sup_closed_form} and Section~\ref{sec:sup_w_bound} respectively.

\begin{lemma}
    Let $\Omega \in \mathbb{R}^{(N - \Theta - k) \times (N - \Theta - k)}$ be the diagonal matrix with $\Omega_{i'i'} = \omega_i$ ($i' = i - k$ to be aligned with the indexing of $\omega_i$). For any vector $\mathfrak{l} \in \mathbb{R}^{N - \Theta - k}$, we have the following inequality:
        \begin{equation*}
                1 \geq \frac{\mathfrak{l}^\top  \Omega \mathfrak{l} }{\|\Omega \mathfrak{l}  \|_2\|\mathfrak{l} \|_2}  \geq 
    \min_{i,j \in \mathcal{I}} \frac{2\sqrt{\omega_i\omega_j}}{\sqrt{\omega_j} + \sqrt{\omega_i}},
        \end{equation*}
    A sufficient and necessary condition for $\frac{\mathfrak{l}^\top  \Omega \mathfrak{l} }{\|\Omega \mathfrak{l}  \|_2\|\mathfrak{l} \|_2}$ being 1 for all $\mathfrak{l}$ is to let  $\omega_i$ be the same for all $i \in \mathcal{I}$.
    \label{lemma:sup_closed_form}
\end{lemma}

\begin{lemma}
Assume $\eta_u$ is upper bounded by a small value $\frac{1}{M}$: 
 $\max_{j = 1 ... N_u}(\Vec{\eta}_u)_j = \frac{1}{M}.$\footnote{Such assumption is used to align the magnitude later in the proof between $\Vec{y} \in [0, 1]$ and $\Vec{\eta}_u \in [0, \frac{1}{M}]$ for the value range.} For each indexing pair $i \in \mathcal{I}$ and $i' \in \mathcal{I}$ with order $\omega_i < \omega_{i'}$, we have
    $$ \frac{\omega_{i}}{\omega_{i'}} \gtrsim \frac{\Vec{y}^\top q_i}{\Vec{\eta}_u^\top q_i} / \frac{\Vec{y}^\top q_{i'}}{\Vec{\eta}_u^\top q_{i'}}.$$
    \label{lemma:sup_w_bound}
\end{lemma}

Putting the ingredients together, we can finally derive an analytical lower bound of $\kappa(\Vec{y})$ in Lemma~\ref{lemma:sup_bound_kappa} based on the angle of $\Vec{y}$ / $\Vec{\eta}_u$ to each eigenvector of $A_{uu}$.  
\begin{lemma}
W.o.l.g, we let $\omega > 0$  and assume that $\omega_i > 0, \forall i \in \mathcal{I}$ so that perturbation of $\omega_i$ to $\omega$ to be not significant  enough to change the sign of $\omega$. we have:
     $$\kappa(\Vec{y}) = \cos(\Bar{U}^{\flat\top} \Vec{y},\Bar{\mathfrak{l}}^{\flat}) 
    \gtrsim \min_{i > k, j > k} \frac{2\sqrt{\frac{\Vec{y}^{\top}q_i}{\Vec{\eta}_u^{\top}q_i}\frac{\Vec{y}^{\top}q_j}{\Vec{\eta}_u^{\top}q_j}}}{\frac{\Vec{y}^{\top}q_i}{\Vec{\eta}_u^{\top}q_i}+\frac{\Vec{y}^{\top}q_j}{\Vec{\eta}_u^{\top}q_j}},$$
    \label{lemma:sup_bound_kappa}
\end{lemma}
\begin{proof}
    Recall that 
$$\Bar{u}_i = (\Bar{\sigma}_i I - A_{uu})^{\dag}\Vec{\eta}_u (\Bar{\mathfrak{l}}^{\flat})_i, $$ 
we consider the specific form of $\kappa(\Vec{y})$, 
\begin{align*}
    \kappa(\Vec{y}) &= \cos \left(\bar{U}^{\flat \top} \vec{y}, \Bar{\mathfrak{l}}^{\flat}\right) \\ 
    &= \frac{\sum^{N}_{i=k+1}\omega_i (\Bar{\mathfrak{l}}^{\flat})^2_i}{\sqrt{\sum^{N}_{i=k+1}\omega^2_i(\Bar{\mathfrak{l}}^{\flat})^2_i}\sqrt{\sum^{N}_{i=k+1}(\Bar{\mathfrak{l}}^{\flat})^2_i}} \\ 
    &= \frac{\sum_{i \in \mathcal{I}}\omega_i (\Bar{\mathfrak{l}}^{\flat})^2_i}{\sqrt{\sum_{i \in \mathcal{I}}\omega^2_i(\Bar{\mathfrak{l}}^{\flat})^2_i}\sqrt{\sum_{i \in \mathcal{I}}(\Bar{\mathfrak{l}}^{\flat})^2_i}} \\ 
    &= \frac{\Bar{\mathfrak{l}}^{\prime\top}  \Omega \Bar{\mathfrak{l}}^{\prime} }{\|\Omega \Bar{\mathfrak{l}}^{\prime} \|_2\|\Bar{\mathfrak{l}}^{\prime}\|_2} ,
\end{align*}
where  $\Omega \in \mathbb{R}^{N_u - k - \Theta}$ is a diagonal matrix defined in Lemma~\ref{lemma:sup_closed_form}, and $\Bar{\mathfrak{l}}^{\prime}$ is  defined in Eq.~\eqref{eq:sup_l_def}. According to Lemma~\ref{lemma:sup_closed_form}, we have 

\begin{align*}
    \kappa(\Vec{y}) &= \frac{\Bar{\mathfrak{l}}^{\prime\top}  \Omega \Bar{\mathfrak{l}}^{\prime} }{\|\Omega \Bar{\mathfrak{l}}^{\prime} \|_2\|\Bar{\mathfrak{l}}^{\prime}\|_2} \\
    & \geq 
    \min_{i,j \in \mathcal{I}} \frac{2\sqrt{\omega_i\omega_j}}{\sqrt{\omega_j} + \sqrt{\omega_i}} \\
    & =     \min_{i,j \in \mathcal{I}} \frac{2}{\sqrt{\frac{\omega_j}{\omega_i}} + \sqrt{\frac{\omega_i}{\omega_j}}},
\end{align*}
Then by Lemma~\ref{lemma:sup_w_bound} and by the fact that $\frac{2}{t+\frac{1}{t}}$ is a monotonically increasing function when $t \in (0, 1)$: 
\begin{align*}
    \kappa(\Vec{y}) &\geq \min_{i,j \in \mathcal{I}} \frac{2}{\sqrt{\frac{\omega_j}{\omega_i}} + \sqrt{\frac{\omega_i}{\omega_j}}} \\
    &\gtrsim \min_{i,j \in \mathcal{I}}  \frac{2}{\sqrt{\frac{\Vec{y}^\top q_i}{\Vec{\eta}_u^\top q_i} / \frac{\Vec{y}^\top q_{j}}{\Vec{\eta}_u^\top q_{j}}} + \sqrt{ \frac{\Vec{y}^\top q_{j}}{\Vec{\eta}_u^\top q_{j}}/\frac{\Vec{y}^\top q_i}{\Vec{\eta}_u^\top q_i}}} \\
    &> \min_{i > k, j > k} \frac{2\sqrt{\frac{\Vec{y}^{\top}q_i}{\Vec{\eta}_u^{\top}q_i}\frac{\Vec{y}^{\top}q_j}{\Vec{\eta}_u^{\top}q_j}}}{\frac{\Vec{y}^{\top}q_i}{\Vec{\eta}_u^{\top}q_i}+\frac{\Vec{y}^{\top}q_j}{\Vec{\eta}_u^{\top}q_j}}.
\end{align*}
\end{proof}

\subsubsection{Proof for Lemma~\ref{lemma:sup_closed_form}}
\label{sec:sup_closed_form}
\begin{proof}
    Consider the function $g(\mathfrak{l}) = \frac{\mathfrak{l}^\top  \Omega \mathfrak{l} }{\|\Omega \mathfrak{l}  \|_2\|\mathfrak{l} \|_2}$, the directional derivative $\partial g(\mathfrak{l})/\partial \mathfrak{l}$ is given by: 
        \begin{align*}
            \frac{\partial g(\mathfrak{l})}{\partial \mathfrak{l}} &= \frac{2 \Omega \mathfrak{l} \|\Omega\mathfrak{l}\|_2 \|\mathfrak{l}\|_2 - \Omega^2\mathfrak{l}\frac{\|\mathfrak{l}\|_2}{\|\Omega\mathfrak{l}\|_2} \mathfrak{l}^\top \Omega\mathfrak{l} - \mathfrak{l} \frac{\|\Omega\mathfrak{l}\|_2}{\|\mathfrak{l}\|_2}\mathfrak{l}^\top \Omega\mathfrak{l} }{\|\Omega\mathfrak{l}\|^2_2\|\mathfrak{l}\|^2_2}.
        \end{align*}
    The condition for $\partial g(\mathfrak{l})/\partial \mathfrak{l} = 0$ is 
    \begin{align*}
        2\Omega\mathfrak{l}  = \Omega^2 \mathfrak{l} \frac{\mathfrak{l}^\top \Omega\mathfrak{l}}{\|\Omega\mathfrak{l}\|^2_2} + \mathfrak{l}  \frac{\mathfrak{l}^\top \Omega\mathfrak{l}}{\|\mathfrak{l}\|^2_2}.
    \end{align*}
    Note that the first condition to satisfy this equation is to let $\mathfrak{l}$ as the eigenvectors of $2\Omega - \Omega^2 \frac{\mathfrak{l}^\top \Omega\mathfrak{l}}{\|\Omega\mathfrak{l}\|^2_2}$ which is a diagonal matrix. Then one of the solutions sets is $\mathfrak{l} = c\*e_j$ where $c$ is any non-zero scalar value and $\*e_j$ is the unit vector with $j$-th value 1 and 0 elsewhere. Note that this solution set corresponds to the maximum value of $g(\mathfrak{l})$ which is 1. We are then looking into the local minimum value of  $g(\mathfrak{l})$ by another solution set. We consider another solution set by considering the following matrix as deficiency:  
        \begin{align*}
        \Gamma \triangleq 2\Omega - \Omega^2 \frac{\mathfrak{l}^\top \Omega\mathfrak{l}}{\|\Omega\mathfrak{l}\|^2_2} - \frac{\mathfrak{l}^\top \Omega\mathfrak{l}} {\|\mathfrak{l}\|^2_2} I,
    \end{align*}
    where $\mathfrak{l}$ lies in the null space of this matrix.     If we let $\varrho = \frac{\|\mathfrak{l}\|_2}{\|\Omega\mathfrak{l}\| _2}$, we have: 
    \begin{align*}
        \Gamma = 2\Omega - \varrho g(\hat{\mathfrak{l}}) \Omega^2  - \varrho^{-1} g(\hat{\mathfrak{l}})  I
    \end{align*}
    and $$\Gamma_{i'i'} = 2\omega_i - \varrho g(\hat{\mathfrak{l}}) \omega_i^2 - \varrho^{-1} g(\hat{\mathfrak{l}}),$$
    where $i'$ is indexed starting from 1 and $i$ is indexed starting from $k$.    Note that $\Gamma_{i'i'}$ only has two zero roots. If we consider all $\omega_i$(s) in $\Omega$ to be different,  $\Gamma$ can have at most two zero values in the diagonal. Let $\omega_a, \omega_b$ as two roots of  $2\omega - \varrho g(\hat{\mathfrak{l}}) \omega^2 - \varrho^{-1} g(\hat{\mathfrak{l}})$, we have: 
    $$\varrho\omega_a + (\varrho\omega_a)^{-1} = \varrho\omega_b + (\varrho\omega_b)^{-1} = \frac{2}{g(\hat{\mathfrak{l}})}$$
    $$\varrho = \frac{\sqrt{\omega_b}}{\sqrt{\omega_a}}, g(\hat{\mathfrak{l}}) = \frac{2}{\sqrt{\frac{\omega_b}{\omega_a}} + \sqrt{\frac{\omega_a}{\omega_b}}}, $$
    which corresponds to one local minimal with the indexing pair $(a,b)$. By enumerating all the indexing pairs, we have the global minimum of $g(\mathfrak{l})$: 
     $$ g(\mathfrak{l}^*) = \min_{i,j \in \mathcal{I}} \frac{2\sqrt{\omega_i\omega_j}}{\sqrt{\omega_j} + \sqrt{\omega_i}}.$$
    Note that when some $\omega_i$, $\omega_j$ are identical, this is a special case where the local minimum is equal to the maximum 1. Therefore a sufficient and necessary condition for $g(\mathfrak{l}) = 1$ is to let $\omega_i$ be the same for all $i \in \mathcal{I}$.
    
\end{proof}

\subsubsection{Proof for Lemma~\ref{lemma:sup_w_bound}}
\label{sec:sup_w_bound}
\begin{proof}
    
We can write $\omega_i$ by $\Tilde{\mathbf{y}}$ and $\Tilde{\boldsymbol{\eta}}$  in Definition~\ref{def:sup_y_eta}: 
\begin{align*}
    \omega_i &= \Vec{y}^{\top}(\Bar{\sigma}_i I - A_{uu})^{\dag}\Vec{\eta}_u \\ 
    &= \sum_{j\in \mathcal{I}} \frac{(\Vec{y}^\top q_j)(\Vec{\eta}_u^\top q_j)}{\Bar{\sigma}_i - d_j} + \sum^{N_u}_{j = N - \Theta + 1} \frac{(\Vec{y}^\top q_j)(\Vec{\eta}_u^\top q_j)}{\Bar{\sigma}_i} \\
    &= \sum_{j\in \mathcal{I}} \frac{\Tilde{\mathbf{y}}_j\Tilde{\boldsymbol{\eta}}_j}{\Bar{\sigma}_i - d_j} + \frac{1}{\Bar{\sigma}_i}\sum^{N_u}_{j = N - \Theta + 1} \Tilde{\mathbf{y}}_j\Tilde{\boldsymbol{\eta}}_j.
\end{align*}
We then look into the value of $\Bar{\sigma}_i$ by solving the eigenvalue problem:  
\begin{align*}
    \left[\begin{array}{cc} 
\eta_l\mathbf{1}_{N_l \times N_l} & \mathbf{1}_{N_l \times 1}\Vec{\eta}_u^{\top} \\ 
\Vec{\eta}_u \mathbf{1}_{1 \times N_l} & A_{uu}
\end{array}\right]\left[\begin{array}{c} 
 \Bar{l}_i \\ 
 \Bar{u}_i 
\end{array}\right] &= \Bar{\sigma}_i \left[\begin{array}{c} 
 \Bar{l}_i \\ 
 \Bar{u}_i 
\end{array}\right] \\ \Longleftrightarrow \ 
\eta_l\mathbf{1}_{N_l \times N_l}  \Bar{l}_i + \mathbf{1}_{N_l \times 1}\Vec{\eta}_u^{\top}  \Bar{u}_i &= \Bar{\sigma}_i \Bar{l}_i \\ \Longleftrightarrow \ 
\mathbf{1}_{N_l \times 1} \eta_l (\Bar{\mathfrak{l}}^{\flat})_i + \mathbf{1}_{N_l \times 1}\Vec{\eta}_u^{\top}  \Bar{u}_i &= \mathbf{1}_{N_l \times 1} \Bar{\sigma}_i \frac{1}{N_l}(\Bar{\mathfrak{l}}^{\flat})_i \\ \Longleftrightarrow \ 
\eta_l (\Bar{\mathfrak{l}}^{\flat})_i + \Vec{\eta}_u^{\top}  \Bar{u}_i &= \Bar{\sigma}_i \frac{1}{N_l}(\Bar{\mathfrak{l}}^{\flat})_i  \\ \quad\quad \Longleftrightarrow
\eta_l (\Bar{\mathfrak{l}}^{\flat})_i + \Vec{\eta}_u^{\top} (\Bar{\sigma}_i I - A_{uu})^{\dag}\Vec{\eta}_u (\Bar{\mathfrak{l}}^{\flat})_i &= \Bar{\sigma}_i \frac{1}{N_l}(\Bar{\mathfrak{l}}^{\flat})_i \\ \Longleftrightarrow \ 
\eta_l + \Vec{\eta}_u^{\top} (\Bar{\sigma}_i I - A_{uu})^{\dag}\Vec{\eta}_u &=  \frac{\Bar{\sigma}_i}{N_l} \\ \Longleftrightarrow \ 
\eta_l + \sum_{j\in \mathcal{I}} \frac{\Tilde{\boldsymbol{\eta}}^2_j}{\Bar{\sigma}_i - d_j}   &=  \frac{\Bar{\sigma}_i}{N_l} \\
\end{align*}

Note that we get a $(|\mathcal{I}| + 1)$-th degree polynomials of $\Bar{\sigma}_i$ with $(|\mathcal{I}| + 1)$ roots. By observation, we see that there is one root significantly large ($\approx N_l \eta_l$) since $N_l$ and other $|\mathcal{I}|$ roots are very close to each $d_j$. Based on this intuition, we approximately view it as a unary quadratic equation: 

$$
\eta_l + \phi_i + \frac{\Tilde{\boldsymbol{\eta}}^2_i}{\Bar{\sigma}_i - d_i}  =  \frac{\Bar{\sigma}_i}{N_l},$$

where we  let $\phi_i \triangleq \sum_{j\in \mathcal{I}, j \neq i} \frac{\Tilde{\boldsymbol{\eta}}^2_j}{\Bar{\sigma}_i - d_j}$. We then proceed by solving this unary quadratic equation by viewing $\phi_i$ as a variable. 
\begin{align*} \Bar{\sigma}_i (\Bar{\sigma}_i - d_i) &= N_l \eta_l (\Bar{\sigma}_i - d_i) + N_l\phi_i (\Bar{\sigma}_i - d_i) + N_l \Tilde{\boldsymbol{\eta}}^2_i \\ \Longleftrightarrow \ 
\Bar{\sigma}_i^2  &= (d_i + N_l (\eta_l + \phi_i))\Bar{\sigma}_i + N_l (\Tilde{\boldsymbol{\eta}}^2_i -  (\eta_l + \phi_i) d_i ) \\ \Longleftrightarrow \ 
\Bar{\sigma}_i &= \frac{d_i + N_l (\eta_l + \phi_i)}{2} \pm \sqrt{\frac{(d_i + N_l (\eta_l + \phi_i))^2}{4} + N_l (\Tilde{\boldsymbol{\eta}}^2_i -  (\eta_l + \phi_i) d_i )} \\ 
\Longleftrightarrow \ 
\Bar{\sigma}_i &= \frac{d_i + N_l (\eta_l + \phi_i)}{2} \pm \sqrt{\frac{(N_l (\eta_l + \phi_i) - d_i)^2}{4} + N_l \Tilde{\boldsymbol{\eta}}^2_i} \\ 
\Longleftrightarrow \ 
\Bar{\sigma}_i &= \frac{d_i + N_l (\eta_l + \phi_i)}{2} \pm \left(\frac{N_l (\eta_l + \phi_i) - d_i}{2} + \frac{N_l \Tilde{\boldsymbol{\eta}}^2_i}{\frac{N_l (\eta_l + \phi_i) - d_i}{2} + \sqrt{\frac{(N_l (\eta_l + \phi_i) - d_i)^2}{4} + N_l \Tilde{\boldsymbol{\eta}}^2_i}}\right) \\ 
\Longleftrightarrow \ 
\Bar{\sigma}_i &= \frac{d_i + N_l (\eta_l + \phi_i)}{2} \pm \left(\frac{N_l (\eta_l + \phi_i) - d_i}{2} + \frac{1}{\frac{\eta_l + \phi_i - \frac{d_i}{N_l}}{2 \Tilde{\boldsymbol{\eta}}^2_i} + \sqrt{(\frac{\eta_l + \phi_i - \frac{d_i}{N_l}}{2 \Tilde{\boldsymbol{\eta}}^2_i})^2 + 1}}\right) \\ \Longleftrightarrow \ 
\Bar{\sigma}_i &= \frac{d_i + N_l (\eta_l + \phi_i)}{2} \pm \left(\frac{N_l (\eta_l + \phi_i) - d_i}{2} + \frac{ \Tilde{\boldsymbol{\eta}}^2_i}{\eta_l + \phi_i - \frac{d_i}{N_l}} - O((\frac{ \Tilde{\boldsymbol{\eta}}^2_i}{\eta_l + \phi_i})^2)\right)
\end{align*}

Here we see that $\Bar{\sigma}_i$ has two approximated solutions: in the first case, when $\pm$ becomes $+$, $\Bar{\sigma}_i \approx N_l \eta_l$ which is the unique very large solution as we mentioned. Another solution is by picking $\pm$ as $-$, we then have $\Bar{\sigma}_i \approx d_i - \frac{ \Tilde{\boldsymbol{\eta}}^2_i}{\eta_l + \phi_i - \frac{d_i}{N_l}}$. The second case is what we are using in this proof since we are looking at the indexing of $\omega_i$ with $i \in \mathcal{I}$, which is beyond top-$k$. 

For each indexing pair $i$ and $i'$ with order $\omega_i < \omega_{i'}$,
we plug in the solution of $\Bar{\sigma}_i$ and $\Bar{\sigma}_i'$ respectively:  
\begin{align*}
    \frac{\omega_{i}}{\omega_{i'}}  &= \frac{ \sum_{j \in \mathcal{I}} \frac{\tilde{\mathbf{y}}_j \tilde{\boldsymbol{\eta}}_j}{d_j - \bar{\sigma}_{i}} + \frac{1}{\Bar{\sigma}_i'}\sum^{N_u}_{j = N - \Theta + 1} \Tilde{\mathbf{y}}_j\Tilde{\boldsymbol{\eta}}_j}{\sum_{j \in \mathcal{I}} \frac{\tilde{\mathbf{y}}_j \tilde{\boldsymbol{\eta}}_j}{d_j - \bar{\sigma}_{i'}} + \frac{1}{\Bar{\sigma}_i''}\sum^{N_u}_{j = N - \Theta + 1} \Tilde{\mathbf{y}}_j\Tilde{\boldsymbol{\eta}}_j} \\
    &= \frac{ \frac{\tilde{\mathbf{y}}_{i} \tilde{\boldsymbol{\eta}}_{i}}{d_{i} - \bar{\sigma}_{i}} + \sum_{j \in \mathcal{I}, j\neq i} \frac{\tilde{\mathbf{y}}_j \tilde{\boldsymbol{\eta}}_j}{d_j - \bar{\sigma}_{i}}+ \frac{1}{\Bar{\sigma}_i'}\sum^{N_u}_{j = N - \Theta + 1} \Tilde{\mathbf{y}}_j\Tilde{\boldsymbol{\eta}}_j}
    {\frac{\tilde{\mathbf{y}}_{i'} \tilde{\boldsymbol{\eta}}_{i'}}{d_{i'} - \bar{\sigma}_{i'}} + \sum_{j \in \mathcal{I}, j\neq i'} \frac{\tilde{\mathbf{y}}_j \tilde{\boldsymbol{\eta}}_j}{d_j - \bar{\sigma}_{i'}}+ \frac{1}{\Bar{\sigma}_i''}\sum^{N_u}_{j = N - \Theta + 1} \Tilde{\mathbf{y}}_j\Tilde{\boldsymbol{\eta}}_j} \\
    &= \frac{ \frac{\tilde{\mathbf{y}}_{i}}{\tilde{\boldsymbol{\eta}}_{i}} (\eta_l + \phi_{i}) + \tilde{\mathbf{y}}_{i}\tilde{\boldsymbol{\eta}}_{i}(O((\frac{ \Tilde{\boldsymbol{\eta}}^2_{i}}{\eta_l + \phi_i})^2) - O(\frac{1}{N_l}))  + \sum_{j \in \mathcal{I}, j\neq i} \frac{\tilde{\mathbf{y}}_j \tilde{\boldsymbol{\eta}}_j}{d_j - \bar{\sigma}_{i}}+ \frac{1}{\Bar{\sigma}_i'}\sum^{N_u}_{j = N - \Theta + 1} \Tilde{\mathbf{y}}_j\Tilde{\boldsymbol{\eta}}_j}
    { \frac{\tilde{\mathbf{y}}_{i'}}{\tilde{\boldsymbol{\eta}}_{i'}} (\eta_l + \phi_{i'}) + \tilde{\mathbf{y}}_{i'}\tilde{\boldsymbol{\eta}}_{i'}(O((\frac{ \Tilde{\boldsymbol{\eta}}^2_{i'}}{\eta_l + \phi_i'})^2) - O(\frac{1}{N_l}))  + \sum_{j \in \mathcal{I}, j\neq i'} \frac{\tilde{\mathbf{y}}_j \tilde{\boldsymbol{\eta}}_j}{d_j - \bar{\sigma}_{i'}}+ \frac{1}{\Bar{\sigma}_i''}\sum^{N_u}_{j = N - \Theta + 1} \Tilde{\mathbf{y}}_j\Tilde{\boldsymbol{\eta}}_j} \\ 
    &= \frac{ \frac{\tilde{\mathbf{y}}_{i}}{\tilde{\boldsymbol{\eta}}_{i}} \eta_l + \tilde{\mathbf{y}}_{i}\tilde{\boldsymbol{\eta}}_{i}(O((\frac{ \Tilde{\boldsymbol{\eta}}^2_{i}}{\eta_l + \phi_i})^2) - O(\frac{1}{N_l}))  + \sum_{j \in \mathcal{I}, j\neq i} \frac{1}{d_j - \bar{\sigma}_{i}}\tilde{\boldsymbol{\eta}}_j (\tilde{\mathbf{y}}_j + \tilde{\mathbf{y}}_{i}\frac{\tilde{\boldsymbol{\eta}}_j}{\tilde{\boldsymbol{\eta}}_{i}}) + \frac{1}{\Bar{\sigma}_i'}\sum^{N_u}_{j = N - \Theta + 1} \Tilde{\mathbf{y}}_j\Tilde{\boldsymbol{\eta}}_j } 
    {  \frac{\tilde{\mathbf{y}}_{i'}}{\tilde{\boldsymbol{\eta}}_{i'}} \eta_l + \tilde{\mathbf{y}}_{i'}\tilde{\boldsymbol{\eta}}_{i'}(O((\frac{ \Tilde{\boldsymbol{\eta}}^2_{i'}}{\eta_l + \phi_i'})^2) - O(\frac{1}{N_l}))  + \sum_{j \in \mathcal{I}, j\neq i'} \frac{1}{d_j - \bar{\sigma}_{i'}}\tilde{\boldsymbol{\eta}}_j (\tilde{\mathbf{y}}_j + \tilde{\mathbf{y}}_{i'}\frac{\tilde{\boldsymbol{\eta}}_j}{\tilde{\boldsymbol{\eta}}_{i'}})+ \frac{1}{\Bar{\sigma}_i''}\sum^{N_u}_{j = N - \Theta + 1} \Tilde{\mathbf{y}}_j\Tilde{\boldsymbol{\eta}}_j}.
\end{align*}

According to assumption that $\eta_u$ is bounded by $\frac{1}{M}$, we align the magnitude between $\Vec{y}$ and $\Vec{\eta}_u$ by  defining $\Vec{\eta}_u' = M\Vec{\eta}_u$ which is now also in the range of $[0, 1]$ like $\Vec{y}$. Then we also scale the following terms: $\Tilde{\boldsymbol{\eta}}' = M\Tilde{\boldsymbol{\eta}}$. Therefore we can simplify the equation to be: 

\begin{align*}
    \frac{\omega_{i}}{\omega_{i'}} &= \frac{ M\frac{\tilde{\mathbf{y}}_{i}}{\tilde{\boldsymbol{\eta}}'_{i}} \eta_l + \frac{1}{M}\tilde{\mathbf{y}}_{i}\tilde{\boldsymbol{\eta}}'_{i}(O(\frac{1}{M^4}(\frac{ \tilde{\boldsymbol{\eta}}'^2_{i}}{\eta_l + \phi_i})^2) - O(\frac{1}{N_l}))  + \frac{1}{M}\sum_{j \in \mathcal{I}, j\neq i} \frac{1}{d_j - \bar{\sigma}_{i}}\tilde{\boldsymbol{\eta}}'_j (\tilde{\mathbf{y}}_j + \tilde{\mathbf{y}}_{i}\frac{\tilde{\boldsymbol{\eta}}'_j}{\tilde{\boldsymbol{\eta}}'_{i}})+ \frac{1}{M\Bar{\sigma}_i'}\sum^{N_u}_{j = N - \Theta + 1} \Tilde{\mathbf{y}}_j\Tilde{\boldsymbol{\eta}}'_j} 
    {  M \frac{\tilde{\mathbf{y}}_{i'}}{\tilde{\boldsymbol{\eta}}'_{i'}} \eta_l + \frac{1}{M} \tilde{\mathbf{y}}_{i'}\tilde{\boldsymbol{\eta}}'_{i'}(O(\frac{1}{M^4}(\frac{ \tilde{\boldsymbol{\eta}}'^2_{i'}}{\eta_l + \phi_i'})^2) - O(\frac{1}{N_l}))  + \frac{1}{M} \sum_{j \in \mathcal{I}, j\neq i'} \frac{1}{d_j - \bar{\sigma}_{i'}}\tilde{\boldsymbol{\eta}}'_j (\tilde{\mathbf{y}}_j + \tilde{\mathbf{y}}_{i'}\frac{\tilde{\boldsymbol{\eta}}'_j}{\tilde{\boldsymbol{\eta}}'_{i'}}) + \frac{1}{M\Bar{\sigma}_i''}\sum^{N_u}_{j = N - \Theta + 1} \Tilde{\mathbf{y}}_j\Tilde{\boldsymbol{\eta}}'_j} \\
    &= \frac{ \frac{\tilde{\mathbf{y}}_{i}}{\tilde{\boldsymbol{\eta}}'_{i}} \eta_l + \tilde{\mathbf{y}}_{i}\tilde{\boldsymbol{\eta}}'_{i}(O(\frac{1}{M^6}) - O(\frac{1}{M^2 N_l}))  + \frac{1}{M^2}\sum_{j \in \mathcal{I}, j\neq i} \frac{1}{d_j - d_{i} + O(\frac{1}{M^2})}\tilde{\boldsymbol{\eta}}'_j (\tilde{\mathbf{y}}_j + \tilde{\mathbf{y}}_{i}\frac{\tilde{\boldsymbol{\eta}}'_j}{\tilde{\boldsymbol{\eta}}'_{i}})+ \tilde{\mathbf{y}}_{i}\frac{\tilde{\boldsymbol{\eta}}'_j}{\tilde{\boldsymbol{\eta}}'_{i}})+ \frac{1}{M^2\Bar{\sigma}_i'}\sum^{N_u}_{j = N - \Theta + 1} \Tilde{\mathbf{y}}_j\Tilde{\boldsymbol{\eta}}'_j} 
    {  \frac{\tilde{\mathbf{y}}_{i'}}{\tilde{\boldsymbol{\eta}}'_{i'}} \eta_l + \tilde{\mathbf{y}}_{i'}\tilde{\boldsymbol{\eta}}'_{i'}(O(\frac{1}{M^6}) - O(\frac{1}{M^2 N_l}))  + \frac{1}{M^2} \sum_{j \in \mathcal{I}, j\neq i'} \frac{1}{d_j - d_{i'} + O(\frac{1}{M^2})}\tilde{\boldsymbol{\eta}}'_j (\tilde{\mathbf{y}}_j + \tilde{\mathbf{y}}_{i'}\frac{\tilde{\boldsymbol{\eta}}'_j}{\tilde{\boldsymbol{\eta}}'_{i'}})+ \frac{1}{M^2\Bar{\sigma}_i''}\sum^{N_u}_{j = N - \Theta + 1} \Tilde{\mathbf{y}}_j\Tilde{\boldsymbol{\eta}}'_j} \\
    & = \frac{ \frac{\tilde{\mathbf{y}}_{i}}{\tilde{\boldsymbol{\eta}}'_{i}} \eta_l + O(\frac{1}{M^2})} 
    {  \frac{\tilde{\mathbf{y}}_{i'}}{\tilde{\boldsymbol{\eta}}'_{i'}} \eta_l + O(\frac{1}{M^2})},
\end{align*}

where we simply regard the remaining term with a magnitude much smaller than M. Note that M can be viewed as the magnitude gap of $\frac{\max_i(\Vec{y})_i}{\max_i(\Vec{\eta}_u)_i}$. In our case, $\max_i(\Vec{y})_i$ is set to 1. However, one can always multiply $\Vec{y}$ with a large constant to make M significantly large without changing the residual analysis in the main theorem. In summary, we have 
$$ \frac{\omega_{i}}{\omega_{i'}} \gtrsim \frac{ \frac{\tilde{\mathbf{y}}_{i}}{\tilde{\boldsymbol{\eta}}'_{i}} \eta_l } 
    {  \frac{\tilde{\mathbf{y}}_{i'}}{\tilde{\boldsymbol{\eta}}'_{i'}} \eta_l } =  \frac{\tilde{\mathbf{y}}_{i}}{\tilde{\boldsymbol{\eta}}_{i}} / \frac{\tilde{\mathbf{y}}_{i'}}{\tilde{\boldsymbol{\eta}}_{i'}} = \frac{\Vec{y}^\top q_i}{\Vec{\eta}_u^\top q_i} / \frac{\Vec{y}^\top q_{i'}}{\Vec{\eta}_u^\top q_{i'}}.$$

\end{proof}

\newpage
\section{Experimental Details}
\subsection{ Details of Training Configurations}
\label{sec:sup_exp_cifar}
For a fair comparison, we use ResNet-18~\cite{he2016deep} as the backbone for all methods. We add a trainable two-layer MLP projection head that projects the feature from the penultimate layer to an embedding space $\mathbb{R}^{k}$ ($k = 1000$). We use the same data augmentation strategies as SimSiam~\cite{chen2021exploring,haochen2021provable}. We train our model $f(\cdot)$ for 1200 epochs by NCD Spectral Contrastive Loss defined in Eq.~\eqref{eq:def_nscl}.  We set $\alpha=0.0225$ and $\beta=2$.
We use SGD with momentum 0.95 as an optimizer with cosine annealing (lr=0.03), weight decay 5e-4, and batch size 512. 
We also conduct a sensitivity analysis of the hyper-parameters in Figure~\ref{fig:hyper}. The performance comparison for each hyper-parameter is reported by fixing other hyper-parameters. The results suggest that the novel class discovery performance of NSCL is  stable when $\alpha$, $\beta$ in a reasonable range and with different learning rates. 

\nocite{sun2017faster,sun2019adaptive}

\begin{figure}[htb]
    \centering
    \includegraphics[width=0.95\linewidth]{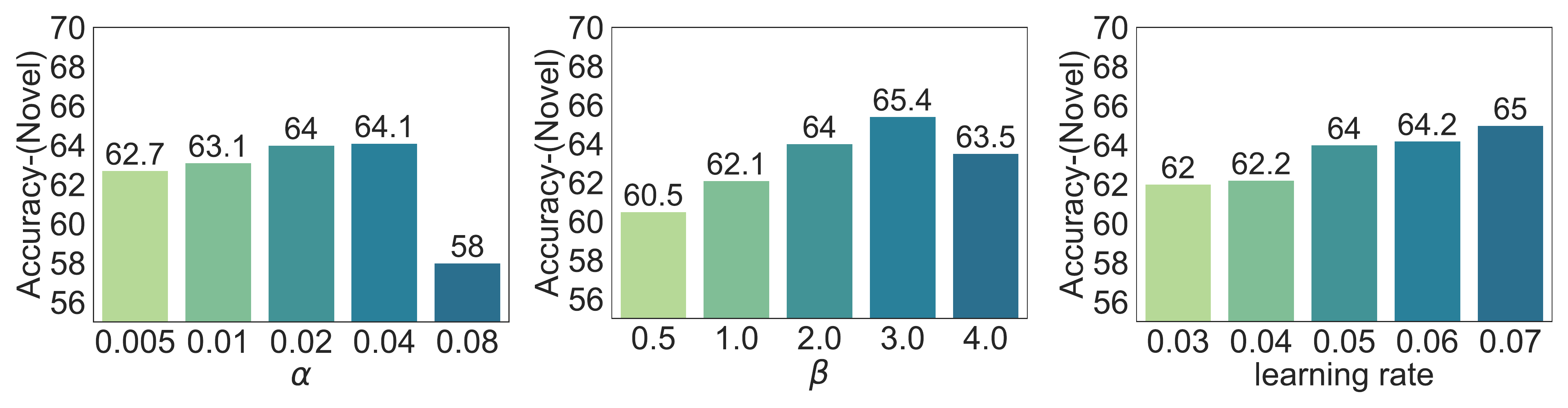}
    \caption{Sensitivity analysis of hyper-parameters $\alpha$, $\beta$, and learning rate. We use the training split of CIFAR-100-50/50, and report the novel class accuracy.}
    \vspace{-0.4cm}
    \label{fig:hyper}
\end{figure}

\subsection{Experimental Details of Toy Example}
\label{sec:sup_exp_vis}
\textbf{Recap of set up}. In Section~\ref{sec:theory_setup} we consider a toy example that helps illustrate the core idea of our theoretical findings. Specifically, the example aims to cluster 3D objects of different colors and shapes, generated by a 3D rendering software~\cite{johnson2017clevr} with user-defined properties including colors, shape, size, position, etc. 

In what follows, we define two data configurations and corresponding graphs, where the labeled data is correlated  with the attribute of unlabeled data (\textbf{case 1}) vs. not (\textbf{case 2}). For both cases, we have an unlabeled dataset containing red/blue cubes/spheres as: 

$$\mathcal{X}_u \triangleq \{X_{\textcolor{red}{\cube{0.6}}, \textcolor{red}{c_1}}, X_{\textcolor{red}{\sphere{0.5}{red}}, \textcolor{red}{c_1}}, X_{\textcolor{blue}{\cube{0.6}}, \textcolor{blue}{c_2}}, X_{\textcolor{blue}{\sphere{0.5}{blue}}, \textcolor{blue}{c_2}}\}.$$

In the first case, we let the labeled data $\mathcal{X}_{l}^{\text{case 1}}$ be strongly correlated with the target class (red color) in unlabeled data:
$$\mathcal{X}_{l}^{\text{case 1}} \triangleq \{X_{\textcolor{red}{\cylinder{0.4}}, \textcolor{red}{c_1}}\} (\text{red cylinder}).$$  
In the second case, we use gray cylinders which have no overlap in either shape and color: 
$$\mathcal{X}_{l}^{\text{case 2}} \triangleq \{X_{\textcolor{gray}{\cylinder{0.4}}, \textcolor{gray}{c_3}}\}  (\text{gray cylinder}).$$ 
Putting it together, our entire training dataset is 
$\mathcal{X}^{\text{case 1}} = \mathcal{X}_{l}^{\text{case 1}} \cup \mathcal{X}_{u}$ or  $\mathcal{X}^{\text{case 2}} = \mathcal{X}_{l}^\text{case 2} \cup \mathcal{X}_{u}$.

\textbf{Experimental details for Figure~\ref{fig:toy_vis}}. For training, we rendered 2500 samples for each type of data (4 types in $\mathcal{X}_u$ and 1 type in $\mathcal{X}_l$). In total, we have 12500 samples for both $\mathcal{X}^{\text{case 1}}$ and $\mathcal{X}^{\text{case 2}}$. For training, we use the same data augmentation strategy as in SimSiam~\cite{chen2021exploring}. We use ResNet18 and train the model for 40 epochs (sufficient for convergence) with a fixed learning rate of 0.005, using NSCL defined in Eq.~\eqref{eq:def_nscl}.  We set $\alpha=0.04$ and $\beta=1$, respectively. Our visualization is by PyTorch implementation of UMAP~\cite{umap}, with parameters $(\texttt{n\_neighbors=30,                min\_dist=1.5, spread=2, metric=euclidean})$.

\end{document}